\documentclass[sn-mathphys,Numbered]{sn-jnl}% Math and Physical Sciences Reference Style
\usepackage{graphicx}%
\usepackage{multirow}%
\usepackage{amsmath,amssymb,amsfonts}%
\usepackage{amsthm}%
\usepackage{mathrsfs}%
\usepackage[title]{appendix}%
\usepackage{xcolor}%
\usepackage{textcomp}%
\usepackage{manyfoot}%
\usepackage{booktabs}%
\usepackage{algorithmic}%
\usepackage{algorithm}%
\usepackage{listings}%
\usepackage{enumitem}
\theoremstyle{thmstyleone}%
\newtheorem{theorem}{Theorem}%  meant for continuous numbers
\theoremstyle{thmstyletwo}%

\theoremstyle{thmstylethree}%

\usepackage{diagbox}

\newcommand{\ty}{\mathbf{y}}
\newcommand{\tyk}{{y^k}}
\newcommand{\tzj}{{z^j}}
\newcommand{\tz}{\mathbf{z}}

\newcommand{\tc}{\mathbf{c}}

\newcommand{\bty}{\mathbf{y}_i}
\newcommand{\btyo}{\mathbf{y}_i^{\textrm{obs}}}
\newcommand{\btym}{\mathbf{y}_i^{\textrm{miss}}}

\newcommand{\btY}{\mathbf{Y}}

\newcommand{\btC}{\mathbf{C}}

\newcommand{\blambda}{{\lambda}}
\newcommand{\bpi}{{\pi}}

\newcommand{\bmu}{{\mu}}
\newcommand{\bpsi}{{\psi}}

\newcommand{\btheta}{\theta }

\newcommand{\PP}{\mathbb{P}}

\newcommand{\argmax}[1]{\arg\underset{#1}{\max}\ }

\newcommand{\ie}{\textit{i}.\textit{e}., }

\usepackage{float}
\usepackage{subcaption}
\usepackage{lscape}
\usepackage{rotating}

\begin{document}

\title{Model-based Clustering \\
with Missing Not At Random Data}


\author{Aude Sportisse, Matthieu Marbac, Fabien Laporte, Gilles Celeux,  Claire Boyer,  Julie Josse, Christophe Biernacki}

%\author*[1,2]{\fnm{First} \sur{Author}}\email{iauthor@gmail.com}



%\affil*[1]{\orgdiv{Department}, \orgname{Organization}, \orgaddress{\street{Street}, \city{City}, \postcode{100190}, \state{State}, \country{Country}}}

%\affil[2]{\orgdiv{Department}, \orgname{Organization}, \orgaddress{\street{Street}, \city{City}, \postcode{10587}, \state{State}, \country{Country}}}

%\affil[3]{\orgdiv{Department}, \orgname{Organization}, \orgaddress{\street{Street}, \city{City}, \postcode{610101}, \state{State}, \country{Country}}}


\abstract{

{Model-based unsupervised learning, as any learning task, stalls as soon as missing data occurs. This is even more true when the missing data are informative, or said missing not at random (MNAR). In this paper, we propose model-based clustering algorithms designed to handle very general types of missing data, including MNAR data.}
To do so, we introduce a mixture model for different types of data (continuous, count, categorical and mixed) to jointly model the data distribution and the MNAR mechanism, remaining vigilant to the relative degrees of freedom of each. Several MNAR models are discussed, for which the cause of the missingness can depend on both the values of the missing variable themselves and on the class membership. However, we focus on a specific MNAR model, called MNAR$z$, for which the missingness only depends on the class membership. We first underline its ease of estimation, by showing that the statistical inference can be carried out on the data matrix concatenated with the missing mask considering finally a standard MAR mechanism.
Consequently, we propose to perform clustering using the Expectation Maximization algorithm, specially developed for this simplified reinterpretation.
Finally, {we assess the numerical performances of the proposed methods } on synthetic data and on the real medical registry TraumaBase$^{\mbox{\normalsize{\textregistered}}}$ as well.

}




\maketitle

%\tableofcontents

\section{Introduction}\label{sec_intro}

%Clustering remains a pivotal tool for readable analysis of large datasets, offering a summary of datasets {by grouping observations}. In particular, the model-based paradigm \citep{mclachlan1988mixture,bouveyron2019model} allows to perform clustering, {by providing interpretable models that are valuable to understand the connections between the constructed clusters and the features in play}. {This parametric framework provides a certain plasticity by handling high-dimensionality problems \citep{bouveyron2007high,bouveyron2014model}, mixed datasets \citep{marbac2017model}, or even time series and dependent data \citep{ramoni2002bayesian,xiong2004time}.} The counterpart to performing this multifaceted model-based clustering is the modeling work involved to design mixture models appropriate to the data structure.

Clustering remains a crucial tool for the comprehensive analysis of large datasets, providing a concise summary through the grouping of observations. Notably, the model-based paradigm \citep{mclachlan1988mixture, bouveyron2019model} facilitates clustering by yielding interpretable models that enhance our understanding of the relationships between the formed clusters and the features at play. This parametric framework exhibits flexibility in handling high-dimensionality problems \citep{bouveyron2007high, bouveyron2014model}, mixed datasets \citep{marbac2017model}, and even time series and dependent data \citep{ramoni2002bayesian, xiong2004time}.
However, the challenge accompanying this multifaceted model-based clustering lies in the requisite modeling efforts to design mixture models tailored to the data structure.

In large-scale data analysis, the problem of missing data is ubiquitous, data collection being never perfect (\textit{e.g.} machines which fail, non-responses in a study). Classical approaches for dealing with missing data consist of {working on a} complete dataset \citep{little2019statistical}, either by using only complete individuals, or by imputing missing values. However, both methods can cause huge problems in the analysis, either by reducing too drastically the dataset to a possibly biased subsample, or by distorting the distribution of the completed samples, respectively.
%Both strategies are, moreover, only preprocessing steps, not specifically designed for the final clustering task. Alternatively, one can consider likelihood approaches, using, for example, Expectation Maximization (EM) type algorithms \citep{dempster1977maximum}. We {adopt} such an approach in this paper, {to make model-based clustering handle informative missing data} in an efficient way. 
Furthermore, it's crucial to note that both of these strategies are essentially preprocessing steps and are not specifically tailored for the final clustering task. Alternatively, one can explore likelihood-based approaches, employing methods like Expectation Maximization (EM) algorithms \citep{dempster1977maximum}. In this paper, we adopt such an approach to enable model-based clustering to effectively handle informative missing data in an efficient manner.

{We assume that the missing data are missing not at random (MNAR) \citep{rubin1976inference,ibrahim2001missing,mohan2018estimation}. More specifically, we consider that the cause of the missingness can be explained by the membership to a class, which is not observed, and we call this specific MNAR model MNAR$z$. As the MNAR mechanism is neither ignorable for the density estimation (parameters estimation), nor for the clustering (partition estimation), dealing with such data does require the specific modeling effort for the distribution of the missing-data pattern, indicating where are the missing values in the data. An example of MNAR data includes clinical data collected in emergency situations, where doctors may choose to treat patients suffering from severe trauma before taking measurements. This intrinsically leads to more missing data in this class.}

\paragraph{Related works on clustering despite missing values}

{In order to handle missing values in a model-based clustering framework,} \citet{HuntJorgensen2003} implement the standard EM algorithm \citep{dempster1977maximum} based on the observed likelihood and \citet{serafini2020handling} propose to perform multiple imputations (with Monte Carlo methods) in the E-step. However, both works only consider MCAR data, when the cause of the missingness is completely independent from the data values. 

Different clustering methods have been developed to deal with MNAR mechanisms. {In a partition-based framework,} \citet{chi2016k} propose an extension of $k$-means clustering for missing data, called $k$-Pod, without requiring the missing-data pattern to be modelled. However, like $k$-means clustering, the $k$-Pod algorithm relies on strong assumptions as equal proportions between clusters. \citet{de2020clustering} perform clustering via a semiparametric mixture model using the pattern-mixture approach to formulate the joint distribution, which makes the method not suitable for estimating the density parameters or imputing missing values. 
For longitudinal data,   \citet{beunckens2008latent} and \citet{kuha2018latent} jointly model the measurements and the dropout process by using an extension of the shared-parameter model.%, which is a specific approach to deal with MNAR mechanisms, by assuming that both the data and the dropout process depend on shared latent variables. 

%For MNAR data, {beyond the clustering task,} the main challenge to overcome consists in proving the identifiability of the parameters of the data and the missing-data pattern distributions. In particular, \citet{Molenberghs08} prove that identifiability does not hold when the models are not fixed, \ie when there is no prior information on the type of distribution for the missing-data pattern. 
%For fixed models, \citet{miao2016identifiability} provide identifiability results of Gaussian mixture and t-mixture models with MNAR data. However, their identifiability results are restricted to specific missing scenarios in a univariate case (one variable), and no estimation strategy is proposed. 

\paragraph{Contributions} 
We first present the considered MNAR$z$ mechanism, in which the missingness depends on class membership, within the context of unsupervised classification based on mixture models for different types of data (continuous, count, categorical, and mixed). We then demonstrate that, under MNAR$z$, statistical inference can be conducted on the augmented matrix formed from the concatenation of the data matrix and the missing-data pattern (binary mask indicating where are the missing values) by considering a missing at random (MAR) mechanism instead. In other words, the cause of the missingness can be
explained by the observed variables, making it ignorable.  This offers a significant advantage, as there is no need to model the missing-data mechanism in such cases. This also provides theoretical insights into an approach commonly used in practice without theoretical foundations, wherein working on the augmented data matrix under a MAR assumption is usually suggested to facilitate efficient learning, despite a more complex underlying missing mechanism \citep{josse2019consistency}.
%Furthermore, we provide identifiability of the mixture model parameters and missingness process parameters under certain conditions, including the data type and the link functions governing the missingness mechanism distribution. 
We then propose an EM algorithm, which has been implemented and is available for reproducibility\footnote{The code is available on \url{https://github.com/AudeSportisse/Clustering-MNAR/tree/main}.}. 
More general MNAR models are also discussed, for which the cause of the missingness can depend on the values of the missing variable themselves, as well as identifiability and estimation strategies. However, they are empirically shown to be well apprehended by the more simple MNARz model. 
%A section is then dedicated to more general variants of MNAR$z$, including cases where the cause of the missingness can also depend on the values of the missing variable themselves. 
Finally, we assess the numerical performance of our method on synthetic data and the real medical registry TraumaBase$^{\mbox{\normalsize{\textregistered}}}$.

\section{Missing-data in model-based clustering}\label{sec_setting}

\label{sec:MBC}

%\subsection{MNAR data}

Set the dataset $Y=(\ty_1 |\ldots | \ty_n)^T$ consisting of $n$ individuals, where each observation $\ty_i=(y_{i1},\ldots,y_{id})^T$ belongs to a space $\mathcal{Y}$, depending on the type of data, defined by $d$ features. 
The pattern of missing data is denoted by $C = (\tc_1 | \ldots | \tc_n)^T\in \{0 ,1\}^{n\times d}$, with $\tc_i=(\tc_{i1},\ldots,\tc_{id})^T\in\{0,1\}^d$: $\tc_{ij}=1$ indicates that the value $y_{ij}$ is missing and $\tc_{ij}=0$ otherwise. The values of the observed (resp.\ missing) variables for individual $i$ are denoted by $\ty_i^{\mathrm{obs}}$ (resp.\ $\ty_i^{\mathrm{mis}}$). {The objective of clustering is to estimate an unknown partition $Z=(\tz_1 |\ldots |\tz_n)^T \in \{0 ,1\}^{n\times K}$ that groups the full dataset $Y$ into $K$ classes, with $\tz_i=(z_{i1},\ldots,z_{iK})^T\in\{0,1\}^K$ and where $z_{ik}=1$ if $\ty_i$ belongs to cluster $k$, $z_{ik}=0$ otherwise. Consequently, in a clustering context, the missing data are not only the values $\ty^{\mathrm{mis}}_i$ but also the partition labels~$\tz_i$.}

\subsection{Mixture models}

Mixture models allow for clustering by modeling the distribution of the observed data $(\btyo,\tc_i)$. 
{Assuming an underlying mixture model with $K$ components, } the probability distribution function (pdf) of the couple $(\ty_i,\tc_i)$ reads as
\begin{equation}\label{eq:mixture}
f(\ty_i,\tc_i;\btheta) = \sum_{k=1}^K \pi_k f_k(\ty_i;\blambda_k)f_k(\tc_i\mid\ty_i;\psi_k),
\end{equation}
where  $\theta=(\gamma,\psi)$ gathers all the model parameters, $\gamma=(\pi,\blambda)$  groups the parameters related to the marginal distribution of $\ty_i$, $\pi=(\pi_1,\ldots,\pi_K)$ is the vector of proportions with $\sum_{k=1}^K \pi_k=1$ and $\pi_k>0$ for all $k\in\{1,\ldots,K\}$.
Given $\blambda=(\blambda_1,\ldots,\blambda_K)$, $f_k(\cdot \, ;\blambda_k)$ is the pdf of the $k$-th component parameterized by $\blambda_k$, $\psi=(\psi_1,\ldots,\psi_K)$ groups the parameters of the missingness mechanisms and $f_k(\tc_i\mid\ty_i;\psi_k)$ is the pdf related to the missingness mechanism under component $k$ (\ie $f_k(\tc_i\mid\ty_i;\psi_k)=f(\tc_i\mid\ty_i,z_{ik}=1;\psi_k)$). In many cases, the parameter $\psi$ is interpreted as a nuisance parameter. However, when the mechanism is not ignorable, i.e.\ can not be ignored when performing inferences for $\lambda$, we need to consider the whole parameter $\theta$ to achieve  clustering, since the pdf of the observed data is
\begin{equation}\label{eq:pdfobserved}
f(\btyo,\tc_i;\btheta) = \int f(\ty_i,\tc_i;\btheta)d\btym.
\end{equation}

Different types of pdf $f_k(\cdot \, ;\blambda_k)$ can be considered, depending on the types of features at hand.
Thus, if $\ty_i$ is a vector of continuous variables, the pdf of a $d$-variate Gaussian distribution \citep{mclachlan1988mixture,banfield1993model} can be considered for $f_k(\ty_i;\blambda_k)$ and thus $\blambda_k$ groups the mean vector and the covariance matrix. 
Moreover, if some components of $\ty_i$ are discrete or categorical, the latent class model (see \cite{geweke1994alternative,mcparland2016model}) defining $f_k(\ty_i;\blambda_k)=\prod_{j=1}^d f_{kj}(y_{ij};\blambda_{kj})$ can be used, with $\lambda_k=(\lambda_{k1},\dots,\lambda_{kd})$. In such a case, $f_{kj}$ could be the pdf of a Poisson (resp.~multinomial) distribution with parameter $\lambda_{kj}$ if $y_{ij}$ is an integer (resp.~categorical) variable. {The choice of the modeling for the missingness mechanism (\ie the distribution $f_k(\tc_i\mid\ty_i;\psi_k)$) is discussed in the following.}

To formulate the joint distribution of the data and the missing-data pattern, we consider in this paper the {selection model} \citep{heckman1979sample}, which factorizes it into the product of the marginal data density and the missing-data mechanism \eqref{eq:mixture}. This approach has the great advantage of allowing imputation of the missing values and density estimation throughout the parameter estimation of the mixture model. Another approach, called the pattern-mixture model \citep{little1993pattern}, can be used, involving the product of the marginal density of $C$ and the conditional density of $Y$ given $C$; it has been considered by \citep{du2023clustering} 
for a clustering purpose.

\subsection{The MNAR$z$ model}\label{sec:MNARz}

To handle MNAR data in selection models, the distribution of the missing-data pattern given the data and the partition should be specified. We consider that the elements of $\tc_i$ are conditionally independent given $(\ty_i,\tz_i)$. By the categorical nature of the mask $\tc_i$, this independence assumption is a quite natural hypothesis in the context of clustering \citep{de2020clustering,chi2016k}. 

In the MNAR$z$ model, we consider that the only effect of
missingness is on the class membership $k$, being the same for all variables. More specifically, the conditional distribution of $c_{ij}$ given $(\ty_i,\tz_i)$ is assumed to be a generalized linear model with link function $\rho$, so that finally
\begin{equation}\label{eq:MNARz}
f_k(\tc_i\mid\ty_i;\psi_k) = \prod_{j=1}^d \rho(\alpha_{k})^{c_{ij}}\left(1-\rho(\alpha_{k})\right)^{1-c_{ij}},
\end{equation}
where $\psi_k=\alpha_k$, in this case. The MNAR$z$ model is the simplest of the MNAR models we can propose (see Section \ref{sec_beyondMNARz} for more general ones). Roughly speaking, this model assumes that the proportion of missing values can vary among the clusters. Although MNAR$z$ does not directly involve  $\ty_i$ in its ground definition~(\ref{eq:MNARz}), it does not mean that the pattern $\tc_i$ does not depend on $\ty_i$ since $\tz_i$ depends itself on $\ty_i$; see Figure~\ref{fig:MNARz.y} for an illustration. %and this can be theoretically observed by noting through the expression $$\PP(\tc_i \mid \ty_i;\btheta) = \sum_{k=1}^K \PP(\tc_i \mid z_{ik}=1;\bpsi_k)\PP(z_{ik}=1 \mid \ty_i;\blambda_k)\neq \PP(\tc_i;\btheta).$$
\begin{center}
\begin{figure}
%\centerline{\includegraphics[scale=0.5]{MNARlatent}}
\centerline{\includegraphics[scale=0.4]{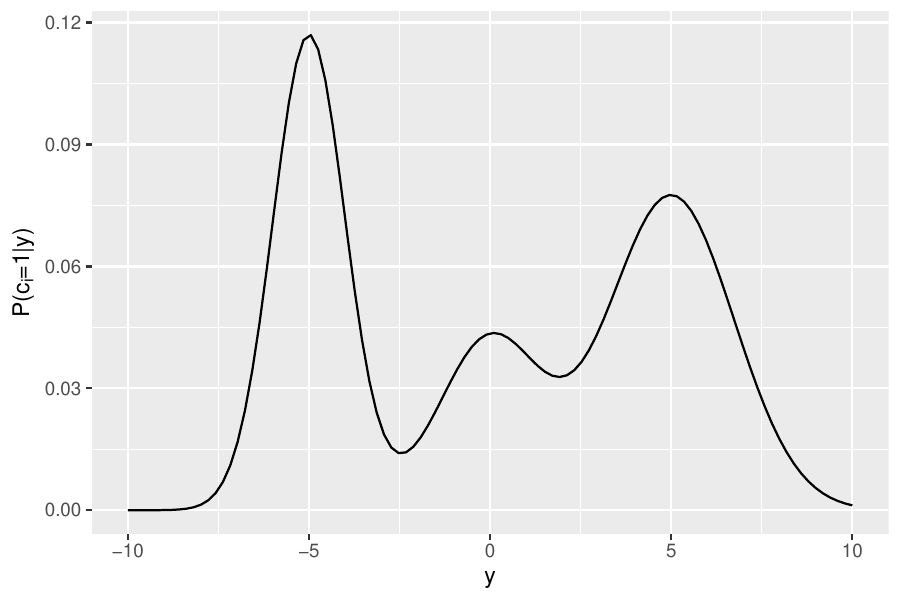}}
\caption{Illustration of dependency between $\tc_i$ and $\ty_i$ in a MNAR$z$ model  by drawing $\PP(c_{i} \mid \ty_i ; \pi,\lambda,\bpsi)$ for
a three component univariate Gaussian model with mixing proportions $\pi_1=\pi_2=0.3$ and $\pi_3=0.4$, with centers $\mu_1=\mu_3=-5$ and $\mu_2=0$, and with variances $\sigma_k^2=k$ ($k\in\{1,2,3\}$). The MNAR$z$ parameters are fixed to $\alpha_1=2$, $\alpha_2=0$ and $\alpha_3=1$.}\label{fig:MNARz.y}
\end{figure}
\end{center}

\section{Proposal}

\subsection{Reinterpretation of the MNAR$z$ model as a MAR strategy}\label{sec_theory}

%\subsection{Reinterpretation of the MNAR$z$ model as a MAR strategy}

Interestingly, the MNAR$z$ model can be turned into a MAR-like one by working on the augmented matrix formed from the concatenation of the data matrix and the missing-
data pattern. %, which is a practice commonly encountered in the machine learning community \citep{josse2019consistency}, \replace{}{without necessarily leaning on theoretical foundations.}
{This is the purpose of the next theorem, proven in  Appendix~\ref{app:proofMAR}.}
% In Theorem~\ref{propo:MAR.MNARz}, we prove that the mixture model associated to this augmented dataset $\tilde \btY^{\mathrm{obs}}$ 
% with a MAR missing mechanism is equivalent to the mixture model for $\btY^{\mathrm{obs}}$ given in \eqref{eq:mixture} assuming a MNAR$z$ or MNAR$\tzj$ model for $\btC$. 
%The proof of this proposition is given in
%Appendix~\ref{un}. 
%For simplicity this proposition is particularized to maximum likelihood estimate, but it could be easily generalized to a large family of other relevant estimation strategies.

\begin{theorem}\label{propo:MAR.MNARz}
Consider the augmented dataset $(\tilde \ty_1^{\mathrm{obs}},\dots,\tilde \ty_n^{\mathrm{obs}})$, $\tilde\ty_i^{\mathrm{obs}}=(\ty_i^{\mathrm{obs}},\tc_i)$ for $i \in \{1,\dots,n\}$. Assume that all $\tilde\ty_i^{\mathrm{obs}}$ arise i.i.d.\ from the mixture model with a MAR mechanism 
\begin{equation} 
\tilde f(\tilde \ty_i^\mathrm{obs}; \btheta) = \sum_{k=1}^K \pi_k f_k(\ty_i^\mathrm{obs};\blambda_k) \prod_{j=1}^d \rho(\alpha_{k})^{c_{ij}} \left(1-\rho(\alpha_{k})\right)^{1- c_{ij}}. \label{eq:mixture2}
\end{equation}
Then for fixed parameter $\btheta=(\pi,\lambda,\bpsi)$, the mixture model for $\tilde \ty_i^\mathrm{obs}$ is the same as the distribution for $\ty_i^\mathrm{obs}$ with the mixture model~(\ref{eq:mixture}) under the MNAR$z$ mechanism~(\ref{eq:MNARz}). 
\end{theorem}
\medskip

Theorem \ref{propo:MAR.MNARz} implies that the maximum likelihood estimate of $\btheta$ is the same considering $\tilde{\ty}_i^{\mathrm{obs}}$ under the MAR assumption and $\ty_i^{\mathrm{obs}}$ under the MNAR$z$ assumption~(\ref{eq:MNARz}).
This implies that if the mechanism is MNAR$z$, an (EM) algorithm designed for MAR data can be used on the augmented data set instead, capitalizing on efficient implementations dedicated to such a well-studied setting (see Section \ref{sec:simu}). 
In fact, Theorem \ref{propo:MAR.MNARz} is the first theoretical result in unsupervised learning in line with the intui\-tion developed in \citep{josse2019consistency}
for supervised learning and in \citep{sportisse2020imputation} for estimation in low-rank models, that working with MAR strategies on the data set augmented by the missing pattern can actually tackle certain types of MNAR settings. 

Furthermore, the identifiability of the model parameters when considering MNAR$z$ data follows directly from this reinterpretation as a MAR model. Indeed, identifiability in the complete case implies identifiability when MCAR or MAR values occur.

\subsection{Associated EM algorithm}\label{sec_estimation}

{Assuming identifiability, we estimate parameters via likelihood maximization using}  the EM algorithm specifically designed for Gaussian, Poisson, multinomial and mixed data with MNAR data. Details of the algorithm are given in Appendix \ref{app:details_algo}. {Assuming that the number $K$ of clusters is known (its choice in practice is discussed in Section \ref{sec:simu}) and that the samples $(y_i,z_i,c_i)_{i=1,\dots,n}$ are i.i.d., the complete-data log-likelihood can be written as}
\begin{equation}
%\begin{split}
\ell_{\mbox{\tiny comp}}(\theta;Y,Z,C) = \sum_{i=1}^n \sum_{k=1}^K z_{ik}\log\left(\pi_k f_k(\ty_i;\blambda_k)f_k(\tc_i\mid \ty_i;\bpsi_k)\right).
\label{LLComp}
%&=\sum_{i=1}^n \sum_{k=1}^K z_{ik}\left[\log(\pi_k f_k(\ty_i;\btheta) + \log(\PP(\tc_i\mid \ty_i,z_{ik}=1;\bpsi))\right]
%\end{split}
\end{equation}

The EM algorithm \citep{dempster1977maximum} is an iterative algorithm that permits to maximize the likelihood function under missingness. Initialized at the point $\theta^{[0]}$, its iteration $[r]$ consists, at the E-step, in computing the expectation of the complete-data log-likelihood $Q(\theta;\theta^{[r-1]})=\mathbb{E}_{\theta^{[r-1]}}\left[\ell_{\mbox{\tiny comp}}(\theta;Y,Z,C)\mid \btY^{\mathrm{obs}}, \btC \right]$, then, at the M-step, updating the parameters by maximizing this function $\theta^{[r]}=\argmax\theta Q(\theta;\theta^{[r-1]})$. Note that 
$$
Q(\theta;\theta^{[r-1]})=\sum_{i=1}^n \sum_{k=1}^K t_{ik}(\theta^{[r-1]})\left[ \log(\pi_k) +  \tau_y(\lambda_k;\ty_i^{\mathrm{obs}},\tc_i,\theta^{[r-1]}) +  \tau_c(\psi_k;\ty_i^{\mathrm{obs}},\tc_i,\theta^{[r-1]}) \right],
$$
\begin{align*}
    \textrm{where} \quad &t_{ik}(\theta^{[r-1]})=\frac{1}{f(\btyo,\tc_i;\btheta^{[r-1]})}\int \pi_k^{[r-1]} f_k(\ty_i;\blambda_k^{[r-1]})f_k(\tc_i\mid\ty_i;\psi_k^{[r-1]})d\btym, \\
    &\tau_y(\lambda_k;\ty_i^{\mathrm{obs}},\tc_i,\theta^{[r-1]})=\mathbb{E}_{\theta^{[r-1]}}\left[\log f_k(y_i;\lambda_k)\mid \ty_i^{\mathrm{obs}},\tc_i,z_{ik}=1\right], \\
    &\tau_c(\psi_k;\ty_i^{\mathrm{obs}},\tc_i,
    \theta^{[r-1]})=\mathbb{E}_{\theta^{[r-1]}}\left[\log f_k(c_i\mid y_i;\bpsi_k)\mid \ty_i^{\mathrm{obs}},\tc_i,z_{ik}=1\right].
\end{align*}

 %\jj{le mot em a deja ete utilisé mais la ref que la?}\as{non je la remets juste ici encore?} 
%consists of iterating the following two steps, starting from an initial parameter value $(\pi^0,\btheta^0,\bpsi^0)$ and  until a stopping criterion is met (e.g.\ a given maximum iteration value $r\leq r_{\mbox{\tiny max}}$): the algorithm alternates between the 
Thus, the iteration $[r]$ of the EM algorithm is defined by
\begin{itemize}
\item {\bf E-step:} Computation of 
$t_{ik}(\theta^{[r-1]}),  \tau_y(\lambda_k;\ty_i^{\mathrm{obs}},\tc_i,\theta^{[r-1]}) \text{ and } \tau_c(\psi_k;\ty_i^{\mathrm{obs}},\tc_i,\theta^{[r-1]}).$
%$Q(\pi,\btheta,\bpsi;\pi^r,\btheta^r,\bpsi^r)$ which is the expected complete log-likelihood $\ell_{\mbox{\tiny comp}}$ knowing the observed data and a current value of the parameters. %This quantity can be decomposed into two parts (see Appendix \ref{app:details_algo} for the full computation) as follows
%with
%\begin{align}
%\label{eq:qy}
%Q_\ty(\pi,\btheta;\pi^r,\btheta^r) &= \sum_{i=1}^n \sum_{k=1}^K (\tau_{ik})^r \log(\pi_k) + \sum_{i=1}^n \sum_{k=1}^K (\tau_{ik})^r E_{i\ty}^r(\btheta), \\
%\label{eq:qc}
%Q_\tc(\bpsi;\bpsi^r) &=\sum_{i=1}^n \sum_{k=1}^K (\tau_{ik})^r E_{i\tc}^r(\bpsi),
%\end{align}
%where for $i=1,\ldots,n$ and $k=1,\ldots,K$,
%\begin{eqnarray}
%\label{eq:eiy}
%E_{i\ty}^r(\btheta)&=&\mathbb{E}\left[\log( f_k(\ty_i;\btheta_k))\mid \ty_i^{\mathrm{obs}},{z_{ik}=1},\tc_i;\btheta^r,\bpsi^r\right], \\
%\label{eq:eic}
%E_{i\tc}^r(\bpsi)&=&\mathbb{E}\left[\log(\PP(\tc_i\mid \ty_i,z_{ik}=1;\bpsi))\mid \ty_i^{\mathrm{obs}},{z_{ik}=1},\tc_i;\btheta^r,\bpsi^r\right], \\
%\label{eq:tauik}
%(\tau_{ik})^r &=& \PP(z_{ik}=1 \mid  \ty_i^{\mathrm{obs}},\tc_i;\btheta^r,\bpsi^r,\pi^r) \;\; \propto \;\; \pi^r_k f_k(\ty_i^{\mathrm{obs}};\theta^r_k)\PP(\tc_i\mid \ty_i^{\mathrm{obs}},z_{ik}=1;\bpsi^r).
%\end{eqnarray}

\item {\bf M-step:} Updating the parameters 
\begin{align*}
    \blambda_k^{[r]}&=\argmax {\lambda_k} \sum_{i=1}^nt_{ik}(\theta^{[r-1]})\tau_y(\lambda_k;\ty_i^{\mathrm{obs}},\tc_i,\theta^{[r-1]}) \\
    \pi^{[r]}_k&=\frac{1}{n}\sum_{i=1}^n t_{ik}(\theta^{[r-1]}) \\
    \bpsi_k^{[r]}&=\argmax {\bpsi_k} \sum_{i=1}^nt_{ik}(\theta^{[r-1]}) \tau_c(\psi_k;\ty_i^{\mathrm{obs}},\tc_i,\theta^{[r-1]}).
\end{align*}
%Maximization over $\pi$, $\btheta$ and $\bpsi$ of $Q(\pi,\btheta,\bpsi;\pi^r,\btheta^r,\bpsi^r)$, by respectively maximizing  $Q_\ty(\pi,\btheta;\pi^r,\btheta^r)$ w.r.t.\ $(\pi,\btheta)$ and $Q_\tc(\bpsi;\bpsi^r)$  w.r.t.\ $\bpsi$.  This step leads to the parameters $\pi^{r+1}$, $\btheta^{r+1}$ and $\bpsi^{r+1}$.
\end{itemize}

{The E-step requires to be able to integrate the distribution of $\ty_i^\mathrm{mis}$ given $(\ty_i^\mathrm{obs}, z_{ik}=1, \tc_i)$ 
%\cb{pourquoi $z_{ik}=1$?}
and the M-step requires to maximize the resulting function. 
This is straightforward under the MNAR$z$ model, because the effect of the missingness does not depend on $\bty$ (see Appendix \ref{app:EMgauss} for computation details in the case of Gaussian or categorical data). %However, these steps become a delicate issue, when the missingness depends on variables $y$,  such models being generically denoted by MNAR$\tyall$ in the sequel. The computation of $\tau_y$ in particular remains feasible in some cases (for instance when the link function in \eqref{eq:meca} is probit), while to our knowledge, $\tau_c$ and $t_{ik}$ do not admit closed forms.
}

% The difficulty of calculating the quantities defined in the E-step as well as the difficulty of the maximization problem in the M-step, depend on the involved MNAR model. These steps are straightforward with MNAR$z$ and MNAR$\tzj$ models.
% \replace{}{but become delicate when the missingness depends on $y$: such models are generically denoted MNAR$\tyall$ in the sequel. }

%Recall that MNAR$\tz$ and MNAR$\tzj$ are the only ones which guarantee identifiability of the parameters for categorical data. Consequently, for the MNAR$\tyall$ models, only algorithms for continuous or count data have to be described. 

\subsection{Beyond the MNARz model}\label{sec_beyondMNARz}

In Section \ref{sec:MNARz}, we proposed the MNAR$z$ model in \eqref{eq:MNARz}. A more general model, called MNAR$\tyk\tzj$, can be considered, when the effect of missingness is on both the class membership and the variable itself:
\begin{equation}\label{eq:meca}
f_k(\tc_i\mid\ty_i;\psi_k) = \prod_{j=1}^d \left(\rho(\alpha_{kj} + \beta_{kj} y_{ij})\right)^{c_{ij}}\left(1-\rho(\alpha_{kj} + \beta_{kj} y_{ij})\right)^{1-c_{ij}},
\end{equation}
where $\psi_k=(\alpha_{k1},\beta_{k1},\ldots,\alpha_{kK},\beta_{kK})$. The parameter  $\alpha_{kj}$ represents a mean effect of missingness on the $k$-th class membership for the variable $j$  (note that within a same class $k$, $\alpha_{kj}$ is not necessarily equal to $\alpha_{kj'}$ for $j\neq j'$). The parameter  $\beta_{kj}$ represents the direct effect of missingness on the variable $j$ which depends on the class $k$ as well.

Simpler models can also be derived from \eqref{eq:meca} by imposing equal parameters either across the class membership, or across the variables likely to be missing.  
First, we introduce three models, with a lower complexity than \eqref{eq:meca}, that still allow  the probability of being missing to depend on both the variable itself and the class membership:
\begin{itemize}
    \item \underline{MNAR$y\tzj$ model}: when $\beta_{1j}=\ldots=\beta_{Kj},\;\forall j$, the effect of missingness on a variable is the same regardless of the class (while keeping different mean effects $\alpha_{kj}$ on the class membership). 
    \item \underline{MNAR$\tyk z$ model}: when $\alpha_{k1}=\ldots=\alpha_{kd},\;\forall k$, the missingness has a same mean effect on class membership shared by all variables (while allowing different self-masked and class-wise parameters $\beta_{kj}$).
    \item \underline{MNAR$yz$ model}: when $\beta_{1j}=\ldots=\beta_{Kj},\;\forall j \quad  \mbox{ and } \quad \alpha_{k1}=\ldots=\alpha_{kd},\;\forall k$, the effects on a particular variable and on the class membership can be respectively the same for all the classes and for all the variables.
\end{itemize}

\medbreak

Secondly, the probability to be missing can also depend only on the variable itself:
\begin{itemize}
    \item \underline{MNAR$y$ model}\footnote{This is actually a particular case of MNAR mechanims, widely used in practice \citep{mohan2018handling}.}: when $\alpha_{11}=\ldots=\alpha_{1d}=\alpha_{21}=\ldots=\alpha_{Kd} \text{ and } \quad \beta_{1j}=\ldots=\beta_{Kj} \quad \forall j$, the only effect of missingness is thus on the variable $j$, being the same regardless of the class membership. 
\item \underline{MNAR$\tyk$ model}: when $\alpha_{11}=\ldots=\alpha_{1d}=\alpha_{21}=\ldots=\alpha_{Kd}$, the effect of missingness on the variable $j$ depends on the class $k$.
\end{itemize}

\medbreak

Thirdly, the probability to be missing  can also depend only on the class membership, so that the missingness is class-wise only. This is the case of the MNAR$z$ model given in \ref{eq:MNARz}, but we can also consider a slightly more general case:
\begin{itemize}
    \item \underline{MNAR$\tzj$ model}: when $\beta_{kj}=0,\; \forall (k,j)$, the effect of missingness on the class membership $k$ is not the same for all the variables.
\end{itemize}

Finally, the simplest model is the missing completely at random (MCAR) one, characterized by no dependence on variables, neither on class membership, \ie each variable has the same probability of missing, 
\begin{equation} \label{eq:MCAR}
\mbox{MCAR:} \quad   \beta_{kj}=0,\; \forall (k,j)\text{ and } \alpha_{1j}=\ldots=\alpha_{Kj},\; \forall j.
\end{equation}

For each of these MNAR models, we have studied the identifiability and have proposed a specific algorithm (EM or Stochastic EM); all the details are given in the accompanying note available \href{https://github.com/AudeSportisse/Clustering-MNAR/blob/main/Accompagnying_Note_Clustering_MNAR.pdf}{here}. Note that for MNAR data, {beyond the clustering task,} the main challenge to overcome consists in proving the identifiability of the parameters of the data and the missing-data pattern distributions \citep{Molenberghs08}.  The identifiability study
showed that the most general models lead to non-identifiable parameters for categorical data (but the identifiability holds only for the MCAR, MNAR$z$ and MNAR$\tzj$ mechanisms). 

Despite the possibility of defining a large number of MNAR models, we have chosen to focus on the MNARz mechanism, which is a good compromise,
clearly outperforming methods that do not consider MNAR data, while limiting the
computational cost of the estimation in regard of more general MNAR mechanisms. Moreover, the MNAR$z$ model is robust to model mispecification (see Section \ref{sec:simu}, Figure \ref{fig:xp2}). 

%This paper is focused on the MNAR$z$ mechanism, which is a good compromise, clearly outperforming methods that do not consider MNAR data, while limiting the computational cost of the estimation in regard of more general MNAR mechanisms (see Section \ref{sec:simu}).

\section{Numerical experiments on synthetic data}\label{sec:simu}

To assess the quality of the clustering, it is possible to use an information criterion such as the Bayesian Information Criterion (BIC) \citep{Schwarz78} or the Integrated Complete-data Likelihood (ICL) \citep{BCG00}. The BIC criterion is expected to select a relevant mixture model from a density estimation perspective, while the ICL is expected to select a relevant mixture model for a clustering purpose \citep{baudry2015estimation}. Thus, we consider the latter in the following. As the ICL involves an integral which is generally not explicit, we can use an approximate version \citep{baudry2015estimation} that we adapt with missing data. For a model $\mathcal{M}$ with $\nu_\mathcal{M}$ parameters, the ICL reads as
\begin{align*}
\mathrm{ICL}(\mathcal{M})&=\ell(\hat{\theta}_\mathcal{M}; Y^{\mathrm{obs}},C) - \frac{\nu_\mathcal{M}}{2}\log n +  \sum_{i=1}^n\sum_{k=1}^K {z}_{ik}^{\mathrm{MAP}}(\hat{\theta}_\mathcal{M})\log(\mathbb{P}(z_{ik}=1 | \ty_i^{\mathrm{obs}},\tc_i;\hat{\theta}_\mathcal{M})),
\end{align*}
where $\hat{\theta}_{\mathcal{M}}$ is a maximum likelihood estimator, $\ell(\theta;  Y^{\mathrm{obs}},C)$ is the observed log-likelihood, and with

\begin{equation}
\label{eq:zikMAP}
{z}_{ik}^{\mathrm{MAP}}(\theta) = \underset{k \in \{1,\dots,K\}}{\mathrm{argmax}} \mathbb{P}(z_{ik}=1 | \ty_i^{\mathrm{obs}},\tc_i;\theta).
\end{equation}

In addition, the Adjusted Rand Index (ARI) \citep{hubert1985comparing} can be computed between the true partition $Z$ and the estimated one.

\subsection{Leveraging from MNAR data in clustering}
\label{subsec:leveraging}
MNAR data are often considered as a real obstacle for statistical processing. 
Yet, the following numerical experiment illustrates that the MNAR mechanism may help performing the clustering task.
Let us consider a bivariate isotropic Gaussian mixture model with two components and equal mixing proportions, under the MNAR$z$ mechanism \eqref{eq:MNARz} with a probit link function. %and identity covariance matrices, i.e.\ %the observations $\btY \in \mathbb{R}^2$ follow the distribution
%$\btY \sim 0.5\mathcal{N}(\bmu_1,{I}_{2\times 2})+0.5 \mathcal{N}(\bmu_2,{I}_{2\times 2})$.
The difference between the centers of both mixture components is taken as $\Delta_\mu=\bmu_{21}-\bmu_{11}=\bmu_{22}-\bmu_{12}\in\{0.5,1,\ldots,3\}$. %where  for any cluster $ k \in \{1,2\}$, the center $\bmu_{k}=(\bmu_{k1},\bmu_{k2})$ is chosen with equal components  ($\bmu_{k2}=\bmu_{k1}$).
This cluster overlap controls the mixture separation, which can vary from a low separation ($\Delta_\mu=0.5$) to a high separation ($\Delta_\mu=3$). 
%By considering the MNAR$z$ mechanism \eqref{eq:MNARz}, %with $\rho$ the cumulative distribution function of the standard Gaussian. 
We also make the discrepancy between inter-cluster missing proportions $\Delta_{\mathrm{perc}}=|{\mathrm{perc}}_2-{\mathrm{perc}}_1|$, vary in $\{0,0.1,0.2,0.3\}$\footnote{The value $\Delta_{\mathrm{perc}}$ means that if the percentage of missing values in the first cluster is ${\mathrm{perc}}_1$, the percentage of missing values in the second cluster is ${\mathrm{perc}}_2=({\mathrm{perc}}_1+\Delta_{\mathrm{perc}})$.} corresponds to {emphasize} the MNAR evidence: indeed, $\Delta_{\mathrm{perc}}=0$ corresponds to a MCAR model, whereas a high value of $\Delta_{\mathrm{perc}}$ corresponds to a high difference of missing pattern proportions between clusters.
\begin{figure}
\centering
\includegraphics[width=0.6\textwidth]{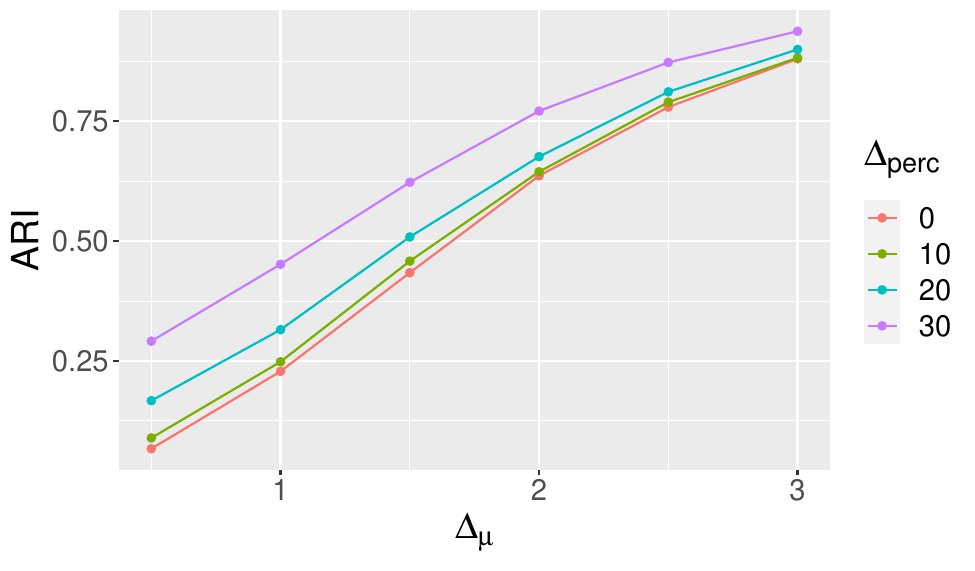}
\caption{\label{fig:Bayes} Relative effect of both the separation strength $\Delta_\mu$ of the mixture component and the MNAR evidence $\Delta_{\mathrm{perc}}$ on theoretical ARI (\textit{e.g.} if $\Delta_\mathrm{perc}=10\%$ (green line), the second class has 10\% more missing values).}
\end{figure}
 For all possible values of $(\Delta_\mu, \Delta_{\mathrm{perc}})$, 15\% missing values are introduced. Figure \ref{fig:Bayes} gives the theoretical ARI (\ie we compute the ARI with the theoretical parameters) as a function of the cluster overlap $\Delta_\mu$ and the MNAR evidence $\Delta_{\mathrm{perc}}$. Although the good classification rate is mainly influenced by center separation $\Delta_\mu$, it also increases with the MNAR evidence $\Delta_{\mathrm{perc}}$. {When classification is difficult because the mixture is not well separated ($\Delta_\mu=0.5$), the fact that the data is MNAR helps clustering:
the theoretical ARI for $\Delta_{\mathrm{perc}}=30$ (MNAR data) is significantly higher than the one for  $\Delta_{\mathrm{perc}}=0$ (MCAR data).} 
This toy example illustrates how clustering can leverage from MNAR values, rather generally considered a true hindrance for any statistical analysis.

\subsection{Generic experiments}
\label{sec:subsec:genericXP}

We consider a Gaussian mixture with three components having unequal proportions ($\pi_1=0.5$, $\pi_2=\pi_3=0.25$) and independent variables:
\begin{equation}\label{eq:simu}
\forall j\in \{1,\dots,d\}, y_{ij}=\delta\sum_{k=1}^3 \varphi_{kj}z_{ik} + \epsilon_{ij},
\end{equation}
with $\epsilon_{ij} \sim \mathcal{N}(0,1)$ the noise term, $\varphi_{k} \in \{0,1\}^d$ and $\delta>0$. {Thus, each entry $y_{ij}$ follows a Gaussian distribution with variance $1$ and mean $\delta\sum_{k=1}^3 \varphi_{kj}z_{ik}$. The values of $\varphi_{kj}$ are arbitrary chosen and highlight the interations between the variable $j$ and the class membership $k$. This formulation allows to control, in any scenario, the theoretical rate of misclassification through the value of $\delta$ (and hence the theoretical ARI).}
We introduce missing values with a MNAR model \eqref{eq:meca}, using a probit link function and control the rates of missingness through the value of $\psi_k$.  For each experiment, the values of $\delta$, $\psi$ and $\phi$ are given in Appendix \ref{app:tab}. When not specified, the simulations have been performed for a theoretical rate of misclassification of $10\%$ and a theoretical missing rate in the whole dataset of 30$\%$. 

\paragraph{Comparison of  MNAR$z$ with other MNAR settings}

We first vary the number of variables ($d=3,6,9$) and consider $n=100$ observations. The missing values are sequentially introduced with a MNAR setting. We compare the method considering the true mechanism (the one used to generate the missing values) with the EM algorithm for MCAR and MNAR$z$ values and the two-step heuristic based on \texttt{Mice}. This latter consists of first imputing the missing values using multiple imputations by chained equations \citep{buuren2010mice} to get $M$ completed datasets. Then, classical model-based clustering is performed on each completed dataset, for which the ARI is computed,
Figure \ref{fig:xp2} shows the boxplot of the ARI for each scenario. First, the methods that consider a MNAR mechanism (MNAR$\star$) always outperform those that consider the MCAR mechanism and the two-step procedure based on \texttt{Mice}. %Note that comparing the MNAR$z$ setting with the real MNAR setting that generated the missing data is difficult, because it is not clear how much the MNAR$z$ setting deviates from the hypothesis (depending on the parameters chosen for the mechanism). 
Finally, the MNAR$z$ model remains a good compromise, clearly outperforming methods that do not consider MNAR data.

\begin{figure}[h]
\centering
\includegraphics[scale=0.5]{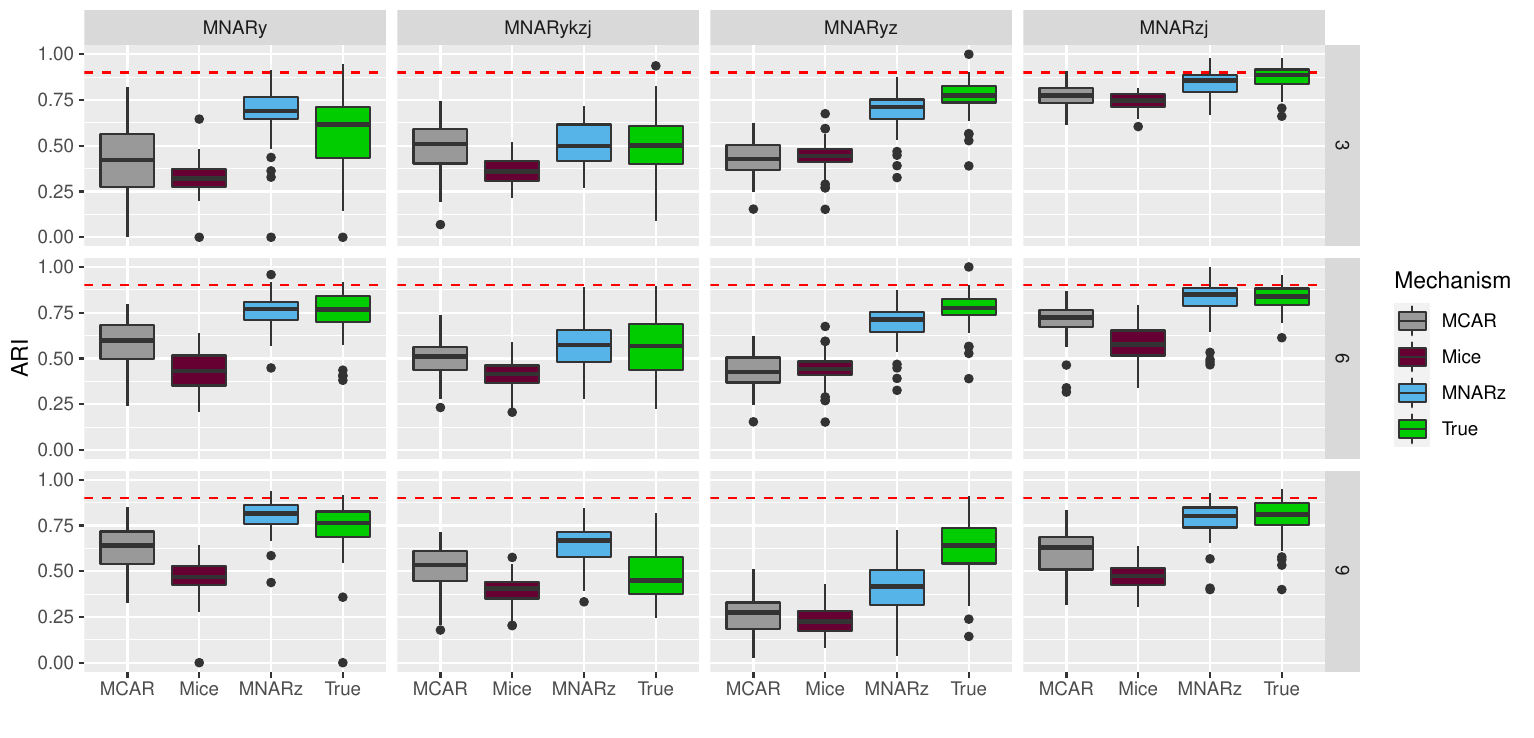}
\caption{\label{fig:xp2}Boxplot of the ARI obtained for 50 samples composed of $d = 3, 6, 9$ variables (rows) and $n=100$ observations. Missing values are introduced with MNAR$y$, MNAR$\tyk\tzj$ MNAR$yz$ or MNAR$\tzj$ settings (columns). The boxplot in green is the one for the algorithm considering the true MNAR$\star$ setting; the boxplot in blue (resp. in gray) is the one for the EM algorithm considering the MNAR$z$ setting (resp. the MCAR setting); the boxplot in red in the two-step heuristic (\texttt{Mice}). {The red dashed line indicates the theoretical ARI.}}
\end{figure}

Moreover, in Appendix \ref{sec:app_xp}, Figure \ref{fig:comptime} and \ref{fig:xp1_MNARz} show the computation times for these numerical experiments; while the MNAR models considering that the probability of being missing depends on the variable itself are computationally very costly, MNAR$z$ model clearly limits the computational cost of the
estimation. 

\paragraph{Focus on the MNAR$z$ mechanism}
Considering the setting \eqref{eq:simu} and under a MNAR$z$ mechanism, we then evaluate the impact of misspecification of the link function (Figure \ref{fig:xp34}(a)), the misspecification of the data distribution (Figure \ref{fig:xp34}(b)) and the percentage of missing values (Figure \ref{fig:xp34}(c)) by comparing the ARI for the MNAR$z$ setting and the MCAR one.

%\paragraph{Impact of the misspecification of the link function} 
In Figure \ref{fig:xp34}(a), {our algorithm always considering a probit function gives the best ARI (outperforming strategies assuming only MCAR data) regardless of the link function (Laplace distribution, logit, probit) used to introduce missing values under a MNAR$z$ model. This highlights the robustness of the MNAR$z$ setting to the link function. 
}
%the missing values are introduced using an MNAR$z$ model with different link functions (the Laplace density distribution, the logit link, and the probit link), whereas the probit one is always considered in the estimation algorithm. 
%The MNAR$z$ setting seems to be robust to the link function. 
In Figure \ref{fig:xp34}(b), we consider a three-component Gaussian mixture with non-diagonal covariance matrices. For each component, the diagonal terms of the covariance matrix are $\Sigma_{ii}=1$ and the other terms $\Sigma_{ij}=\ell, i\neq j$, with $\ell\in \{0,0.1,0.25,0.5\}$, while the algorithms assume $\ell=0$. If the EM algorithm designed for MNAR$z$ data suffers from a huge deviation ($\ell=0.5$) regarding the data distribution, it remains competitive for smaller ones ($\ell=0.1,0.25$). 
%For smaller deviations , the results are still satisfactory and clearly outperform the ones given by the EM algorithm for the MCAR setting. 
Finally, Figure \ref{fig:xp34}(c) shows the boxplots of the ARI for $10\%$, $30\%$ and $50\%$ of missing values in the entire dataset. As the percentage of missing data increases, {the gap between algorithms considering MCAR and MNAR$z$ data is widening, proving the relevancy of our algorithm %considering MNAR$z$ data 
even with high missing-data rates (50\%).} 
In Appendix \ref{sec:app_xp}, we also provide the experiments for a theoretical rate of misclassification of 15\%. Same conclusions hold.
%the difference between the algorithms considering MCAR and MNAR$z$ data is greater. 
%Even if the percentage of missing data has an impact on the algorithm considering MNAR$z$ data, it still gives results close to the theoretical ARI for a missing-data rate of 50\%. 

% \begin{figure}
%      \centering
%      \begin{subfigure}[b]{0.33\textwidth}
%          \centering
%          \includegraphics[width=0.33\textwidth]{fig/Plot_xp3a.pdf}
%          \caption{Impact of the misspecification of the link function}
%          \label{fig:xp3a}
%      \end{subfigure}
%      \hfill
%      \begin{subfigure}[b]{0.33\textwidth}
%          \centering
%          \includegraphics[width=0.33\textwidth]{Plot_xp3b.pdf}
%          \caption{Impact of the misspecification of the data distribution}
%          \label{fig:xp3b}
%      \end{subfigure}
%      \vfill
%     \begin{subfigure}[b]{0.33\textwidth}
%          \centering
%          \includegraphics[width=0.33\textwidth]{Plot_xp4a.pdf}
%          \caption{Impact of the percentage of missing values }
%          \label{fig:xp4a}
%      \end{subfigure}
%         \caption{Boxplot of the ARI obtained for 50 samples composed of $d = 6$ variables. The missing values are introduced using a MNAR$z$ setting.}
%         \label{fig:three graphs}
    
% \end{figure}
\begin{figure}
     \centering
     \begin{tabular}{ccc}
     \includegraphics[width=0.3\textwidth]{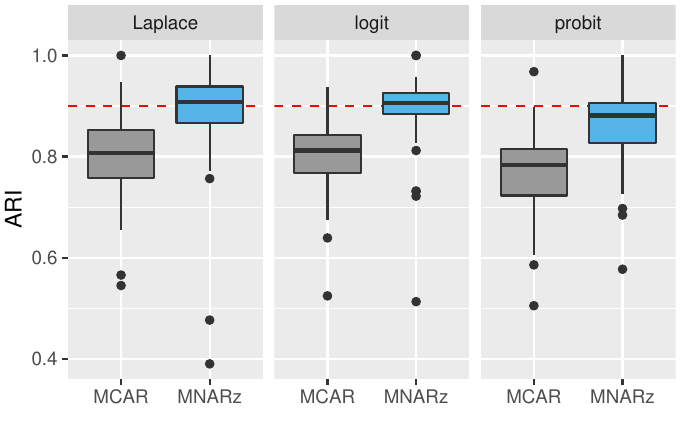} &
     \includegraphics[width=0.33\textwidth]{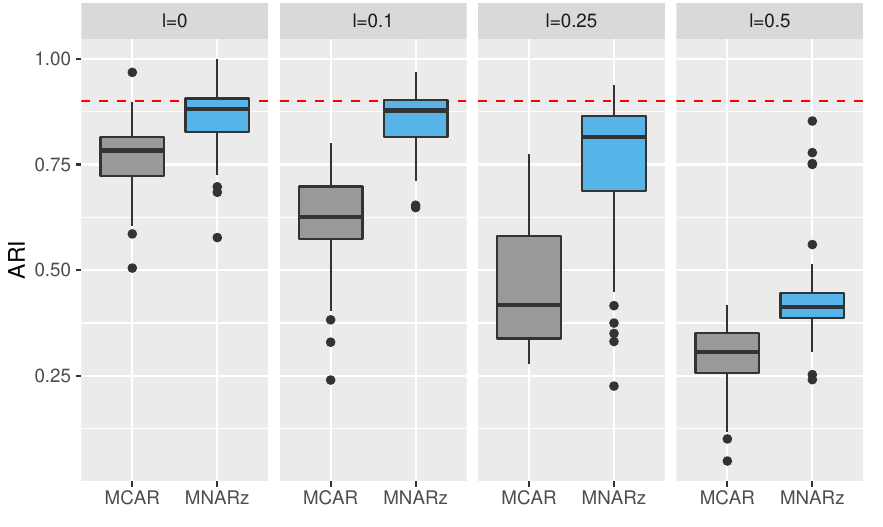} &
     \includegraphics[width=0.3\textwidth]{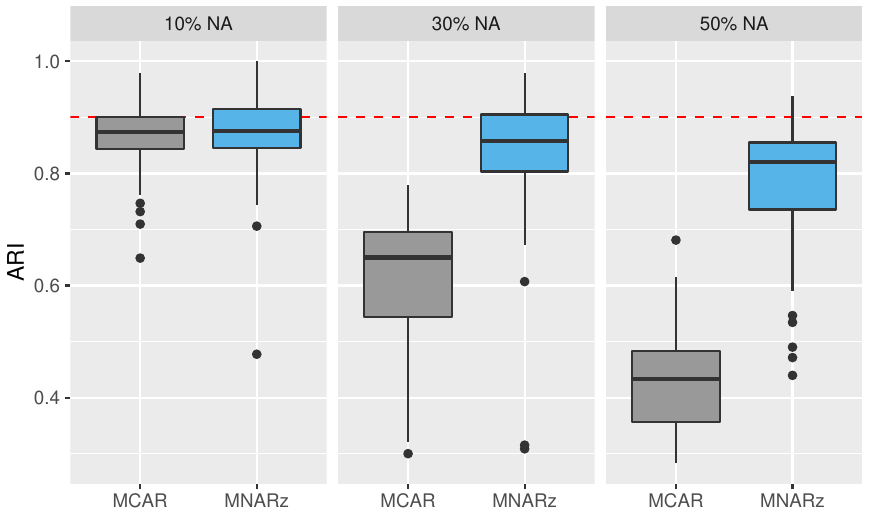} \\
     (a) & (b) & (c)
     \end{tabular}
        \caption{Boxplot of the ARI obtained for 50 samples of dimension $d = 6$ variables. The missing values are introduced using a MNAR$z$ setting. The red dashed line indicates the theoretical ARI.}
        \label{fig:xp34}
\end{figure}

%\paragraph{Impact of the misspecification of the data distribution} 

%\paragraph{Choice of $K$}
When the number $K$ of clusters is not known \emph{a priori}, it can be automatically chosen using the ICL criterion: the idea is to run algorithms with several values for $K$ ($K=1,2,3,4$ here), and to choose the model with the highest resulting ICL. To our knowledge, no method proposes an automatic choice of the number of clusters in unsupervised classification for the two-step heuristics, which is also a major drawback.

\begin{center}
\begin{table}
\caption{\label{tab:choiceK} Proportion of good selections of $K$ ($K=3$) using the ICL criterion for the EM algorithm, over 50 repetitions  ($d=6$).}
\begin{tabular}{l|ll|ll}
\hline
   & \multicolumn{2}{c}{MCAR}  & \multicolumn{2}{c}{MNARz}  \\ \hline
 Sample size $n$   & $100$ & $500$  & $100$ & $500$  \\ \hline
10 \% NA &  94\% & 100\% & 94\% & 100\% \\ \hline
30 \% NA  & 8\%  & 96\% & 56\% & 100\%  \\ \hline
50 \% NA  & 0\% & 0\% & 20\% & 98\% \\ \hline
\end{tabular}
\end{table}
\end{center}
Table \ref{tab:choiceK} gathers the percentages of times (over 50 repetitions) the correct number of classes ($K=3$) is chosen by the ICL criterion for different missing-data rates (10, 30, 50\%) and different sample sizes ($n=100,500$). In any case, the EM algorithm for MNAR$z$ data always outperforms the algorithm for MCAR data in terms of accurate model selection. The EM algorithm for MNAR$z$ data manages also to select the best model despite a high percentage of missing data (50\%) provided that the sample size is large enough ($n=500$).

We also illustrate in Appendix \ref{sec:app_xp} (Figure \ref{fig:xp1_MNARz}) the findings of Theorem \ref{propo:MAR.MNARz}, by comparing  our algorithm considering MCAR or MNAR$z$ data with the algorithm of the \textbf{RMixtComp} package \citep{biernacki2015model} considering MCAR data and using the augmented data matrix ($Y|C$). As expected, both approaches give similar results. 

\section{TraumaBase$^{\mbox{\normalsize{\textregistered}}}$ dataset}\label{sec_realdata}

In this section, we illustrate our approach on a public health application with the  TraumaBase$^{\mbox{\normalsize{\textregistered}}}$ Group (\url{https://www.traumabase.eu/en_US}) on the management of traumatized patients. This dataset contains 41 mixed variables (continuous, quantitative) on $8,248$ polytraumatized patients who suffer from a major trauma (injuries from cycle or car accident). Data have been collected from 15 different hospitals. In this dataset, 11\% of the data are missing and only 1.4\% of the individuals are fully observed. More information on the variables can be found in Appendix \ref{app:realdata}. The purpose of this real data analysis is twofold: (i) we want to know if considering the missingness process has an impact on the estimated partition, (ii) we compare our method with the classical imputation methods in Appendix \ref{app:realdata}.

After discussion with doctors, some variables can be considered to have MNAR values, such as \textit{Shock.index.ph}, which denotes the ratio between heart rate and systolic arterial pressure. In fact, if this rate has a value that indicates that the patient's condition is critical, doctors cannot measure heart rate or systolic arterial pressure in emergency situations. Therefore, we expect that considering a MNAR mechanism can improve the classification.

%In this section, the variables related to the patient death and also the hand-made classifications made by the doctors (one considers 3 groups, the other 4) are not taken into account for running the algorithms, as they were considered too informative for the classification. 
We compare our algorithm designed for the MNAR$z$ data \eqref{eq:MNARz} and the MCAR data \eqref{eq:MCAR}.
Figure \ref{fig:realdataICL} presents the ICL values in the Traumabase dataset for different numbers of classes. {If both algorithms select $K=3$ number of classes, the ICL of the algorithm which considers MNAR$z$ data is 
nonetheless always significantly higher than that of the algorithm for MCAR data.} %In the categorical case, the algorithm for MNAR$z$ data selects $K=10$, whereas the algorithm for MCAR data selects $K=9$. Both algorithms select higher number of classes in the continuous case, $K=20$ classes for MNAR$z$ data and $K=10$ classes for MCAR data. The case where the algorithms select very different numbers of clusters is therefore the continuous case, where the variables are no longer considered independent conditionally to the group membership.  We can assume that the information provided by the covariances is significant in this dataset, which would show the limits of our model in the mixed case, although this assumption of conditional independence is very classical in clustering. Note also that the ICL of the algorithm which considers MNAR$z$ data is always higher than the one of the algorithm for MCAR data. 
{Their corresponding ARI between classifications obtained assuming either MNAR$z$ or MCAR mechanisms is about 0.90.} 
Thus, both partitions are close but not equal, which may reflect the influence of the mechanism. To deepen this issue, we focus on the variable \textit{Shock.index.ph}. Table \ref{tab:realdata_distribvar} and \ref{tab:realdata_probaclass} compare the performances of the algorithm considering MNAR$z$ data with the one considering MCAR data in terms of modelling of the marginal distribution of \textit{Shock.index.ph} and partition estimation. As the values can be compared only up to label swapping, we notice that the minimum values (on the diagonals) are significantly higher than zero, which indicates that there is an influence of the MNAR$z$ mechanism on the modeling of the data and on the classification rules.

\begin{figure}
\centering
\begin{subfigure}[t]{0.43\textwidth}
    \centering
    \includegraphics[width=\textwidth]{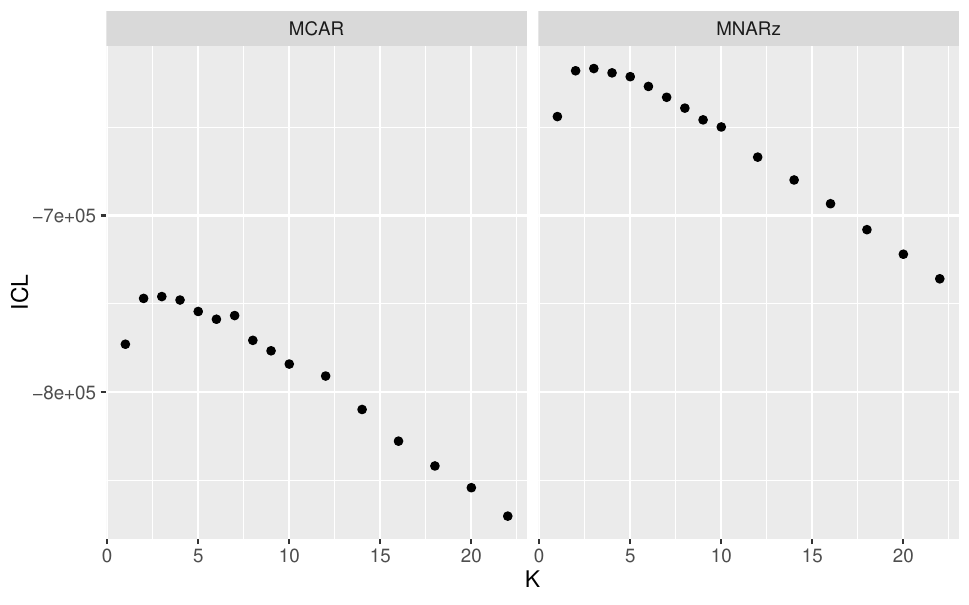}
    \caption{ICL values for different $K$ in the mixed case.}
    \label{fig:realdataICL}
\end{subfigure}
\begin{minipage}{0.49\textwidth}
\vspace*{-2cm}
\begin{subfigure}{\textwidth}
\centering
\scalebox{0.6}{
\begin{tabular}{ccccccc}

\diagbox{MCAR}{MNAR$z$} & 
Class 1 & Class 2 & Class 3 \\ 
\hline
Class 1  & \textbf{0.03} & 0.47 & 0.61 \\ \hline
Class 2 &   0.45   &  \textbf{0.05} & 0.25 \\ \hline
Class 3    &  0.63 & 0.23 & \textbf{0.03} \\ \hline
\end{tabular}
}
\caption{\label{tab:realdata_distribvar} Total variation distance between the marginal distribution of the variable \textit{Shock.index.ph} considering MNAR$z$ data and the one considering MCAR data.}
\end{subfigure}
\begin{subfigure}{\textwidth}
\centering
\scalebox{0.6}{
\begin{tabular}{ccccccc}
\diagbox{MCAR}{MNAR$z$} & 
Class 1 & Class 2 & Class 3 \\ 
\hline
Class 1  & \textbf{2.43} & 26.5  & 37.6 \\ \hline
Class 2 &  26.2    & \textbf{3.40} & 20.1 \\ \hline
Class 3    &  39.3  & 19.2 & \textbf{2.05} \\ \hline
\end{tabular}
}
\caption{\label{tab:realdata_probaclass} Euclidean distance between conditional probabilities of the cluster memberships given observed values of the variable \textit{Shock.index.ph}, considering MNAR$z$ or MCAR data (more details in Appendix \ref{app:realdata}).}
\end{subfigure}
% \vfill
% \begin{subfigure}[b]{0.8\textwidth}
% \scalebox{0.75}{
% \begin{tabular}{ccccccc}

% \diagbox{MCAR}{MNAR$z$} & 
% Class 1 & Class 2 & Class 3 \\ 
% \hline
% Class 1  & \textbf{0.03} & 0.47 & 0.61 \\ \hline
% Class 2 &   0.45   &  \textbf{0.05} & 0.25 \\ \hline
% Class 3    &  0.63 & 0.23 & \textbf{0.03} \\ \hline
% \end{tabular}
% }
% \caption{\label{tab:realdata_distribvar} Total variation distance between the marginal distribution of the variable \textit{Shock.index.ph} considering MNAR$z$ data and the one considering MCAR data.}
% \end{subfigure}
\end{minipage}
\caption{Comparison of the MCAR and MNAR$z$ mechanism on Traumabase dataset.}
\end{figure}

%Moreover, in Table \ref{tab:realdata_classif}, we compute the ARI for the classification assuming MNAR$z$ (resp. MCAR) mechanism and the home-made classification by the doctors considered as a ground truth. %\footnote{We compare to the home-made classification by the doctors which assumes $K=4$ when $K=4$ and the one which assumes $K=3$ for the other cases ($K=2,3,5$).} the ICL (fourth column), the ARI between the classifications obtained assuming MNAR$z$ and MCAR mechanisms (fifth column). 
%Note that the difficulty is that we do not have the \textit{true} classification, because the one made by the doctors are hand-made and they use very few variables (one or two). The low ARIs (always inferior to 10\%) may therefore reflect the fact that we take into account many variables and that our classification is therefore necessarily different from those of physicians who use very few variables. 

%\begin{table}
%    \centering
%    \begin{tabular}{c|c|c}
%     & ARI (MCAR) & ARI (MNAR$z$)  \\
%     \hline
%    $K=2$ & 9,2\% & 9.3\% \\
%    \hline
%    $K=3$  & 9,3\% &  9.3\%\\
%   $K=4$ & 6.2\% & 8.4\% \\
%    $K=5$ &  7.0\% & 8.6\%\\
%    \end{tabular}
%    \caption{Comparison with the home-made classification.}
%    \label{tab:realdata_classif}
%\end{table}

In Appendix \ref{app:realdata}, we assess also the results by using the function \texttt{catdes} of the R package FactoMineR \citep{hussonfacto2008} which allows to see how the cluster of the classification is described by the variables. The three groups described with our algorithm assuming MNAR$z$ data seem to be described by the same characteristics than those given by the doctors: the first group is formed by patients with a higher mortality rate and more severe injuries than the average population and the third group by a lower mortality rate and less serious injuries, whereas the second group may correspond to other cases. Note that this classification was done without using variables related to patient death and that it is quite striking to retrieve the same characteristics. This reinforces the idea that the classification obtained makes sense and may provide other information than the one of the doctors, taking into account more variables.%, because it takes into account all the variables.

\section{Concluding remarks}
\label{sec:conclu}

This paper addresses {model-based} unsupervised learning when MNAR values occur. 
{We propose to cluster individuals via an estimation of the mixture model parameters in play. A by-product of such an approach is that the missing values can be also imputed, once the distribution is estimated.} To this end, we have proposed an approach which embeds MNAR data directly within model-based clustering algorithms, in particular the EM algorithm. We have discussed several possible MNAR specifications. However, the numerical experiment leads us to recommend using algorithms considering a simple missing-data mechanism, the MNAR$z$ mechanism, which models the probability of being missing only depending on the class membership. By its very simplicity, {the latter is} indeed able to straightforwardly deal with any kind of data. 
{In addition to being interpretable (which is especially important for real applications), this MNAR$z$ mechanism can be apprehended as a MAR one on the augmented matrix $[Y|C]$, including the missing-data pattern $C$ (Theorem \ref{propo:MAR.MNARz}). This echoes a widely-used approach in practice, not theoretically studied so far. }

The {seminal} motivation of this work was {clustering patients of the Traumabase dataset, in particular to assist doctors in their medical care.  }
{After a first conclusive application, } there are still key challenges to make this work entirely applicable to real datasets. 
First, if our methodology can be applied to mixed data (categorical/quan\-ti\-ta\-tive),  a straightforward extension of the proposed approach {should be doable to handle variables that are not necessarily of the same type (MCAR, MAR and MNAR variables are indeed often coupled)}. 
{Without any prior help from experts, this actually remains an open question to automatically evaluate the missing type of variables.} {Note however that one can arbitrate between the presented MNAR mechanisms using the ICL criterion, at the price of running multiple times the algorithm for the different MNAR scenarios. %Therefore, without any insight on the MNAR type, we highly recommend to use the MNAR$z$ mechanism, its versatility having already been outlined.}

\newpage

\begin{appendices}

\section{Proof of Theorem \ref{propo:MAR.MNARz}}\label{app:proofMAR}

\begin{proof}[Proof of Theorem~\ref{propo:MAR.MNARz}]
We denote by $(\tilde\tc_1,\ldots,\tilde\tc_n)$ the patterns of missing data associated to the observed data $\tilde\ty^{\mathrm{obs}}$. It is thus the concatenation $\tilde\tc_i=(\tc_i,{{\bf 0}_d})$  of $\tc_i$ with the zero vector ${\bf 0}_d=(0,\ldots,0)$ of length $d$. Since all $c_i$ values are observed in $\tilde\ty_i^{\mathrm{obs}}$, it is the reason why the last $d$ values in $\tilde\tc_i$ are fixed to zero. Then, the MAR assumption indicates that $\PP(\tilde\tc_i \mid \tilde\ty_i,\tz_i ; \zeta)=\PP(\tilde\tc_i \mid \tilde\ty^{\mathrm{obs}}_i;\zeta)$, with $\zeta$ the related parameter. %\footnote{This proof can be enlarged to the non parametric case for $\PP(\tilde\tc_i \mid \tilde\ty^\mathrm{obs}_i)$.}. 
Consequently, using the MAR assumption and the i.i.d. assumption of all uplets $(\tilde\ty_i,\tz_i,\tilde\tc_i)$, the whole likelihood can be decomposed into two likelihoods, one has
\begin{eqnarray*}
%\nonumber
L(\theta, \zeta ; \tilde{Y^{\mathrm{obs}}},C) & = & \prod_{i=1}^n\int f(\tilde{\ty}_i,\tilde{\tc}_i;\btheta,\zeta)d\tilde{\ty}_i^{\mathrm{mis}}\\
& = & \prod_{i=1}^n\int f(\tilde{\ty}_i;\pi,\lambda,\psi)f(\tilde{\tc}_i|\tilde{\ty}_i;\zeta)d\tilde{\ty}_i^{\mathrm{mis}}\\
& = & \prod_{i=1}^n \left[f(\tilde{\tc}_i|\tilde{\ty}_i^\mathrm{obs};\zeta)\; \times \; \int_{{\mathcal Y}^{\mathrm{mis}}_i} f(\tilde{\ty}_i;\pi,\lambda,\psi) d\tilde{\ty}_i^{\mathrm{mis}}\right] \\
& = & \prod_{i=1}^n L(\zeta ; \tilde{\tc}_i \mid \tilde{\ty}_i^{\mathrm{obs}}) \; \times\; \prod_{i=1}^n L(\bpi,\lambda, \psi ; \tilde{\ty}_i^{\mathrm{obs}}). %\label{eq:likelihood.MAR}
\end{eqnarray*}

\begin{equation}
L(\pi,\lambda, \psi,\zeta;\tilde\ty^\mathrm{obs}_i,\tilde\tc_i) = L(\zeta;\tilde\tc_i \mid \tilde\ty^\mathrm{obs}_i) \times L(\pi,\lambda, \psi;\tilde\ty^\mathrm{obs}_i).
\end{equation}
Providing that $(\pi,\lambda,\psi)$ and $\zeta$ are functionally independent (ignorability of the MAR mechanism), the maximum likelihood estimate of $\theta=(\pi,\lambda, \psi)$ is obtained by maximizing only $L(\bpi,\lambda,\psi ; \tilde{\ty}_i^{\mathrm{obs}})$, and does not depend on $L(\zeta;\tilde\tc_i \mid \tilde\ty^\mathrm{obs}_i)$. Finally, by using~(\ref{eq:mixture2}), the observed likelihood $L(\pi,\lambda,\bpsi;\tilde\ty^\mathrm{obs}_i)$ is
\begin{eqnarray}
L(\pi,\lambda,\bpsi;\tilde\ty^\mathrm{obs}_i) & = & \sum_{k=1}^K \pi_k f_k(\ty_i^\mathrm{obs};\lambda_k) \prod_{j=1}^d \rho(\alpha_{k})^{c_{ij}} \rho(\alpha_{k})^{(1- c_{ij})} \\
& = & \sum_{k=1}^K \pi_k f_k(\ty_i^\mathrm{obs};\lambda_k) \prod_{j=1}^d f(c_{ij} \mid z_{ik}=1 ; \bpsi).
\end{eqnarray}
As $\PP(c_{ij} \mid z_{ik}=1 ; \bpsi)$ corresponds to the MNAR$z$ definition~(\ref{eq:MNARz}), the observed likelihood $L(\pi,\lambda,\bpsi;\tilde\ty^\mathrm{obs}_i)$ is equal to the full observed likelihood $L(\pi,\lambda,\bpsi;\ty^\mathrm{obs}_i,\tc_i)$ associated to the MNAR$z$ model, 
\begin{equation*}
     L(\pi,\lambda,\bpsi;\ty^\mathrm{obs}_i,\tc_i)=\sum_{k=1}^K \pi_k f_k(\ty_i^\mathrm{obs};\lambda_k) \prod_{j=1}^d f(c_{ij} \mid z_{ik}=1 ; \bpsi).
\end{equation*}
\end{proof}

\section{Detail on EM algorithm}\label{app:details_algo}

%Here are detailed (classical) formula for calculation of $E_\ty^{[r-1]}(\btheta)$ (E-step) and maximization of $Q_\ty(\btheta;\btheta^{[r-1]},\bpsi^{[r-1]})$ over $\btheta$ (M-step), once the $t_{ik}(\theta^{[r-1]})$s are computed. Related results depend on the mixture model at hand, thus on the kind of features.

The EM algorithm consists on two steps iteratively proceeded: the E-step and M-step. For the E-step, one has
{\small
\begin{align*}
Q(\theta;\theta^{[r-1]})&=\mathbb{E}[\ell_{\mbox{\tiny comp}}(\theta;Y,Z,C)|y_i^\mathrm{obs},c_i;\theta^{[r-1]}]\\
&=\sum_{i=1}^n\sum_{k=1}^K \mathbb{E}\left[z_{ik}\log(\pi_k f_k(\ty_i;\lambda)f(\tc_i\mid \ty_i,z_{ik}=1;\bpsi))\mid y_i^\mathrm{obs},c_i;\pi^{[r-1]},\lambda^{[r-1]},\psi^{[r-1]}\right] \\
%&=\sum_{i=1}^n\sum_{k=1}^K \int_{\mathcal{Y}_i^\mathrm{mis}} \log(\pi_k f_k(\ty_i;\lambda)f(\tc_i\mid \ty_i,z_{ik}=1;\bpsi)) f(\ty_i^\mathrm{mis},z_{ik}=1|\ty_i^{\mathrm{obs}},\tc_i;\pi^{[r-1]},\lambda^{[r-1]},\psi^{[r-1]})d\ty_i^\mathrm{mis} \\
&=\sum_{i=1}^n\sum_{k=1}^K  t_{ik}(\theta^{[r-1]})  \mathbb{E}\left[\log(\pi_k f_k(\ty_i;\lambda)f(\tc_i\mid \ty_i,z_{ik}=1;\bpsi))\mid y_i^\mathrm{obs},c_i,z_{ik}=1;\pi^{[r-1]},\lambda^{[r-1]},\psi^{[r-1]}\right]
\end{align*}
%using $f(\ty_i^\mathrm{mis},z_{ik}=1|\ty_i^{\mathrm{obs}},\tc_i;\theta^{[r-1]})= t_{ik}(\theta^{[r-1]}) f(\ty_i^\mathrm{mis}|\ty_i^{\mathrm{obs}},z_{ik}=1,\tc_i;\theta^{[r-1]})$ 
with $t_{ik}(\theta^{[r-1]}) = f(z_{ik}=1 \mid  \ty_i^{\mathrm{obs}},\tc_i;\theta^{[r-1]})$.}

It leads to the decomposition 

$$
Q(\theta;\theta^{[r-1]})=\sum_{i=1}^n \sum_{k=1}^K t_{ik}(\theta^{[r-1]})\left[ \log(\pi_k) +  \tau_y(\lambda_k;\ty_i^{\mathrm{obs}},\tc_i,\theta^{[r-1]}) +  \tau_c(\psi_k;\ty_i^{\mathrm{obs}},\tc_i,\theta^{[r-1]}) \right],
$$
where the terms involved in this decomposition are now detailed. 

\begin{enumerate}[label=(\alph*)]
%\item \label{eipi} the expectation of the mixing proportion parameters over the missing values given the available information:
%$$E_{i\pi}^{[r-1]}(\pi)=\mathbb{E}\left[\log(\pi_k)\mid \ty_i^{\mathrm{obs}},{z_{ik}=1},\tc_i;\btheta^{[r-1]},\bpsi^{[r-1]},\pi^{[r-1]}\right]$$
\item \label{eiy}the expectation of the data mixture part over the missing values given the available information (\ie the observed data and the indicator pattern), the class membership and the current value of the parameters:  
$$ \tau_y(\lambda_k;\ty_i^{\mathrm{obs}},\tc_i,\theta^{[r-1]})=\mathbb{E}_{\theta^{[r-1]}}\left[\log f_k(y_i;\lambda_k)\mid \ty_i^{\mathrm{obs}},z_{ik}=1,\tc_i\right],$$ 
\item \label{eic} the expectation of the missing mechanism part over the missing values given the available information, the class membership and the current value of the parameters: 
$$\tau_c(\psi_k;\ty_i^{\mathrm{obs}},\tc_i,\theta^{[r-1]})=\mathbb{E}_{\theta^{[r-1]}}\left[\log f_k(c_i\mid y_i;\bpsi_k)\mid \ty_i^{\mathrm{obs}},z_{ik}=1,\tc_i\right].$$
\item \label{tauik} the conditional probability for an observation $i$ to belong to the class $k$ given the available information and the current value of the parameters: $$t_{ik}(\theta^{[r-1]}) = f(z_{ik}=1 \mid  \ty_i^{\mathrm{obs}},\tc_i;\theta^{[r-1]}).$$
\end{enumerate}
Terms \ref{eiy} and \ref{eic} require to integrate over the distribution $f(\ty_{i}^{\mathrm{mis}}\mid \ty_i^{\mathrm{obs}},z_{ik}=1,\tc_i;\theta^{[r-1]})$. 
For Term \ref{eiy}, one has
\begin{align}
\nonumber
&f(\ty_{i}^{\mathrm{mis}}\mid \ty_i^{\mathrm{obs}},z_{ik}=1,\tc_i;\theta^{[r-1]}) \\
&=\frac{f(\ty_{i}^{\mathrm{mis}},\ty_i^\mathrm{obs},z_{ik}=1,\tc_i;\theta^{[r-1]})}{f(\ty_i^\mathrm{obs},z_{ik}=1,\tc_i;\theta^{[r-1]})} \\
\label{eq:astuce_missobs}
&=\frac{f(\tc_i\mid \ty_{i}^{\mathrm{mis}},\ty_i^{\mathrm{obs}},z_{ik}=1;\psi^{[r-1]})f(\ty_{i}^{\mathrm{mis}}, \ty_i^{\mathrm{obs}}, z_{ik}=1;\lambda^{[r-1]})}{\int_{\mathcal{Y}_{i}^\mathrm{mis}} f(\tc_i\mid \ty_{i}^{\mathrm{mis}},\ty_i^{\mathrm{obs}},z_{ik}=1;\psi^{[r-1]})f(\ty_{i}^{\mathrm{mis}}, \ty_i^{\mathrm{obs}}, z_{ik}=1;\lambda^{[r-1]}) d\ty_{i}^\mathrm{mis}}.
\end{align}
Term \ref{tauik} corresponds to the conditional probability for an observation $i$ to arise from the $k$th mixture component with the current values of the model parameter. More particularly, one has
\begin{align}
\nonumber
t_{ik}(\theta^{[r-1]})&=\frac{f(z_{ik}=1,  \ty_i^{\mathrm{obs}},\tc_i;\theta^{[r-1]})}{f(\ty_i^{\mathrm{obs}},\tc_i;\theta^{[r-1]})} \\
\nonumber
&=\frac{f(z_{ik}=1,  \ty_i^{\mathrm{obs}},\tc_i;\theta^{[r-1]})}{\sum_{h=1}^{K}f(z_{ih}=1,\ty_i^{\mathrm{obs}},\tc_i;\theta^{[r-1]})} \\
\nonumber
&=\frac{\pi_k^{[r-1]}f(\ty_i^{\mathrm{obs}}\mid z_{ik}=1;\lambda^{[r-1]}_k)f(\tc_i\mid \ty_i^{\mathrm{obs}},z_{ik}=1;\theta^{[r-1]})}{\sum_{h=1}^{K}\pi_h^{[r-1]})f(\ty_i^{\mathrm{obs}}\mid z_{ih}=1;\lambda^{[r-1]}_h)f(\tc_i\mid \ty_i^{\mathrm{obs}},z_{ih}=1;\theta^{[r-1]})} \\
&=\frac{\pi^{[r-1]}_k f_k(\ty_i^{\mathrm{obs}};\lambda^{[r-1]}_k)f(\tc_i\mid \ty_i^{\mathrm{obs}},z_{ik}=1;\theta^{[r-1]})}{\sum_{h=1}^K \pi^{[r-1]}_h f_h(\ty_i^{\mathrm{obs}};\lambda^{[r-1]}_h)f(\tc_i\mid \ty_i^{\mathrm{obs}},z_{ih}=1;\theta^{[r-1]})}
\label{astuce_tauik}
\end{align}

%%%%% REF PROD_CI
%Note also that
%\begin{equation}
%\label{eq:prod_ci}
%f(\tc_i\mid \ty_i^{\mathrm{obs}},z_{ik}=1;\theta^{[r-1]},\bpsi^{[r-1]})=\prod_{j=1}^{d} f(\tc_{ij}\mid \ty_i^{\mathrm{obs}},z_{ik}=1;\theta^{[r-1]},\bpsi^{[r-1]})
%\end{equation}

\subsection{Gaussian mixture for continuous data}
\label{app:EMgauss}
The pdf $f_k(\ty_i;\lambda) = \phi(\ty_i;\mu_k,\Sigma_k)$ is assumed to be a Gaussian distribution with mean vector $\mu_k$ and covariance matrix $\Sigma_k$. 
First, let us detail the terms of the E-step. Term \ref{eiy} is written as follows:
\begin{multline*}
%\begin{split}
\mathbb{E}\left[\log(\phi(\ty_i;\mu_k,\Sigma_k))\mid \ty_i^{\mathrm{obs}},z_{ik}=1,\tc_i;\theta^{[r-1]}\right] =-\frac{1}{2}\left[n\log(2\pi) + \log((\mid\Sigma_k\mid)) \right]\\
-\frac{1}{2}\mathbb{E}\left[ (\ty_i - \mu_k)^T(\Sigma_k)^{-1}(\ty_i - \mu_k)   \mid \ty_i^{\mathrm{obs}},z_{ik}=1,\tc_i;\theta^{[r-1]}\right].
%\end{split}
%\label{Estep::logY_i}.
\end{multline*}
%We hope to write the last term in \eqref{Estep::logY_i} as $tr(W_k\Sigma_k^{-1})$ where $W_k$ is the empirical covariance matrix times $n_k$. In this case, the update of $\Sigma_k$ will be $\frac{W_k}{n_k}$.\\
This last term could be expressed using the commutativity and linearity of the trace function:
\begin{multline*}
\mathbb{E}\left[ (\ty_i - \mu_k)^T(\Sigma_k)^{-1}(\ty_i - \mu_k)   \mid \ty_i^{\mathrm{obs}},z_{ik}=1,\tc_i;\theta^{[r-1]} \right] \\
= \mbox{tr}(\mathbb{E}\left[ (\ty_i - \mu_k)(\ty_i - \mu_k)^T   \mid \ty_i^{\mathrm{obs}},z_{ik}=1,c_i;\theta^{[r-1]} \right](\Sigma_k)^{-1}).
\end{multline*}
Finally note that only $\mathbb{E}\left[ (\ty_i - \mu_k)(\ty_i - \mu_k)^T   \mid \ty_i^{\mathrm{obs}},z_{ik}=1,\tc_i;\theta^{[r-1]} \right]$ has to be calculated. 

%\paragraph{MNAR$z$ and MNAR$\tzj$ models}
For the MNAR$z$ model, the effect of the missingness is only due to the class membership. %Term \ref{eiy} is the same for both models but \ref{eic} and \ref{tauik} differ. Let us first detail these terms.
\begin{itemize}
    \item
For Term \ref{eiy}, note that $$f(\ty_{i}^{\mathrm{mis}}\mid \ty_i^{\mathrm{obs}},z_{ik}=1,\tc_i;\theta^{[r-1]})=f(\ty_{i}^{\mathrm{mis}}\mid \ty_i^{\mathrm{obs}},z_{ik}=1;\lambda^{[r-1]}),$$
which makes the computation easy. 
Indeed, using \eqref{eq:astuce_missobs}, 
\begin{align*}
&f(\ty_{i}^{\mathrm{mis}}\mid \ty_i^{\mathrm{obs}},z_{ik}=1,\tc_i;\theta^{[r-1]}) \\
&=\frac{\prod_{j=1}^d \rho(\alpha_{k}^{[r-1]})^{c_{ij}}(1-\rho(\alpha_{k}^{[r-1]}))^{1-c_{ij}}f(\ty_{i}^{\mathrm{mis}}, \ty_i^{\mathrm{obs}}, z_{ik}=1;\lambda^{[r-1]})}{\int_{\mathcal{Y}_{i}^\mathrm{mis}} \prod_{j=1}^d \rho(\alpha_{k}^{[r-1]})^{c_{ij}}(1-\rho(\alpha_{k}^{[r-1]}))^{1-c_{ij}}f(\ty_{i}^{\mathrm{mis}}, \ty_i^{\mathrm{obs}}, z_{ik}=1;\lambda^{[r-1]}) d\ty_{i}^\mathrm{mis}} \\
&=\frac{f(\ty_{i}^{\mathrm{mis}}, \ty_i^{\mathrm{obs}}, z_{ik}=1;\lambda^{[r-1]})}{\int_{\mathcal{Y}_{i}^\mathrm{mis}}f(\ty_{i}^{\mathrm{mis}}, \ty_i^{\mathrm{obs}}, z_{ik}=1;\lambda^{[r-1]}) d\ty_{i}^\mathrm{mis}}=f(\ty_{i}^{\mathrm{mis}}\mid \ty_i^{\mathrm{obs}},z_{ik}=1;\lambda^{[r-1]}),
\end{align*}
since $\prod_{j=1}^d \rho(\alpha_{k}^{[r-1]})^{c_{ij}}(1-\rho(\alpha_{k}^{[r-1]}))^{1-c_{ij}}$ does not depend on $\ty_{i}^\mathrm{mis}$ and is simplified with the numerator. 
The law of $(\ty_{i}^{\mathrm{mis}}\mid \ty_i^{\mathrm{obs}},z_{ik}=1)$ is Gaussian (up to a reorganization of the variables associated to individual $i$). Noting that
\begin{align*}
\left(\ty_i\mid z_{ik}=1;\lambda^{[r-1]} \right) &= \left(\left( \begin{array}{c}
\ty_i^{\mathrm{obs}}\\
\ty_i^{\mathrm{mis}}
\end{array}\right)\mid z_{ik}=1;\lambda^{[r-1]} \right)  \\
&\sim \mathcal{N} \left( \left( \begin{array}{c}
(\mu_{ik}^{\mathrm{obs}})^{[r-1]}\\
(\mu_{ik}^{\mathrm{mis}})^{[r-1]}
\end{array}\right) , \left(\begin{array}{cc}
(\Sigma_{ik}^{\mathrm{obs},\mathrm{obs}})^{[r-1]} &(\Sigma_{ik}^{\mathrm{obs},\mathrm{mis}})^{[r-1]}\\
(\Sigma_{ik}^{\mathrm{mis},\mathrm{obs}})^{[r-1]} &(\Sigma_{ik}^{\mathrm{mis},\mathrm{mis}})^{[r-1]}
\end{array}\right)\right), 
\end{align*}
one obtains
\begin{equation}
\label{eq:misgivenobsapp}
\left( \ty_i^{\mathrm{mis}} \mid  \ty_i^{\mathrm{obs}},z_{ik}=1;\lambda^{[r-1]} \right) \sim \mathcal{N}\left((\tilde{\mu}_{ik}^{\mathrm{mis}})^{[r-1]},(\tilde{\Sigma}_{ik}^{\mathrm{mis}})^{[r-1]}\right).
\end{equation}
with $(\tilde{\mu}_{ik}^{\mathrm{mis}})^{[r-1]}$ and $(\tilde{\Sigma}_{ik}^{\mathrm{mis}})^{[r-1]}$ the standard expression of the mean vector and covariance matrix of a conditional Gaussian distribution (see for instance \cite{Anderson2003}) detailed as follows
\begin{align}
\label{eq:misgivenobs_mu}
(\tilde{\mu}_{ik}^{\mathrm{mis}})^{[r-1]} &= ({\mu}_{ik}^{\mathrm{mis}})^{[r-1]} + ({\Sigma}_{ik}^{\mathrm{mis},\mathrm{obs}})^{[r-1]} \left(({\Sigma}_{ik}^{\mathrm{obs},\mathrm{obs}})^{[r-1]}\right)^{-1} \left(\ty_i^{\mathrm{obs}} - ({\mu}_{ik}^{\mathrm{obs}})^{[r-1]} \right),\\
\label{eq:misgivenobs_sigma}
(\tilde{\Sigma}_{ik}^{\mathrm{mis}})^{[r-1]} &= ({\Sigma}_{ik}^{\mathrm{mis},\mathrm{mis}})^{[r-1]} - ({\Sigma}_{ik}^{\mathrm{mis},\mathrm{obs}})^{[r-1]} \left(({\Sigma}_{ik}^{\mathrm{obs},\mathrm{obs}})^{[r-1]}\right)^{-1} ({\Sigma}_{ik}^{\mathrm{obs},\mathrm{mis}})^{[r-1]}.
\end{align}
Note also that we have
\begin{equation*}
(\ty_i - \mu_k)(\ty_i - \mu_k)^T = \left( \begin{array}{cc}
(\ty_i^{\mathrm{obs}} - \mu_{ik}^{\mathrm{obs}})^T(\ty_i^{\mathrm{obs}} - \mu_{ik}^{\mathrm{obs}}) &(\ty_i^{\mathrm{obs}} - \mu_{ik}^{\mathrm{obs}})^T(\ty_i^{\mathrm{mis}} - \mu_{ik}^{\mathrm{mis}})\\
(\ty_i^{\mathrm{mis}} - \mu_{ik}^{\mathrm{mis}})^T(\ty_i^{\mathrm{obs}} - \mu_{ik}^{\mathrm{obs}}) &(\ty_i^{\mathrm{mis}} - \mu_{ik}^{\mathrm{mis}})^T(\ty_i^{\mathrm{mis}} - \mu_{ik}^{\mathrm{mis}})
\end{array} \right).
\end{equation*}

Therefore, the expected value of each block for the current parameter value is
\begin{equation*}
\begin{split}
\mathbb{E}\left[ (\ty_i^{\mathrm{obs}} - \mu_{ik}^{\mathrm{obs}})^T(\ty_i^{\mathrm{obs}} - \mu_{ik}^{\mathrm{obs}})  \mid \ty_i^{\mathrm{obs}},z_{ik}=1;\lambda^{[r-1]} \right] &= (\ty_i^{\mathrm{obs}} - \mu_{ik}^{\mathrm{obs}})^T(\ty_i^{\mathrm{obs}} - \mu_{ik}^{\mathrm{obs}})\\
\mathbb{E}\left[ (\ty_i^{\mathrm{obs}} - \mu_{ik}^{\mathrm{obs}})^T(\ty_i^{\mathrm{mis}} - \mu_{ik}^{\mathrm{mis}})  \mid \ty_i^{\mathrm{obs}},z_{ik}=1;\lambda^{[r-1]}\right] &= (\ty_i^{\mathrm{obs}} - \mu_{ik}^{\mathrm{obs}})^T((\tilde{\mu}_{ik}^{\mathrm{mis}})^{[r-1]} - \mu_{ik}^{\mathrm{mis}})\\
\mathbb{E}\left[ (\ty_i^{\mathrm{mis}} - \mu_{ik}^{\mathrm{mis}})^T(\ty_i^{\mathrm{mis}} - \mu_{ik}^{\mathrm{mis}})   \mid \ty_i^{\mathrm{obs}},z_{ik}=1;\lambda^{[r-1]} \right] &=
((\tilde{\mu}_{ik}^{\mathrm{mis}} )^{[r-1]}- \mu_{ik}^{\mathrm{mis}})^T((\tilde{\mu}_{ik}^{\mathrm{mis}})^{[r-1]} - \mu_{ik}^{\mathrm{mis}})+(\tilde{\Sigma}_{ik}^{\mathrm{mis}})^{[r-1]}.
\end{split}
\end{equation*}
\item For Term \ref{eic}, $f(\tc_i\mid \ty_i,z_{ik}=1;\bpsi)$ is independent of $\ty_i$, which implies 
\begin{equation}
\label{eq:termb_gauss_MNARz}
\log(f(\tc_i\mid z_{ik}=1;\bpsi))=\sum_{j=1}^d c_{ij}\log\rho(\alpha_{k})+ (1-c_{ij})\log(1-\rho(\alpha_{k}))
\end{equation}
\item For Term \ref{tauik}, one first remark that
$$\mathbb{P}(\tc_i\mid \ty_i^\mathrm{obs},z_{ik}=1;\theta^{[r-1]})=\prod_{j=1}^d \mathbb{P}(c_{ij}=1\mid \ty_i^\mathrm{obs},z_{ik}=1;\theta^{[r-1]})^{c_{ij}}\mathbb{P}(c_{ij}=0\mid y_i^\mathrm{obs},z_{ik}=1;\theta^{[r-1]})^{1-c_{ij}}.$$
In particular, for MNAR$z$, by independence of $\ty_i$, one has $$\mathbb{P}(c_{ij}=1\mid \ty_i^\mathrm{obs},z_{ik}=1;\theta^{[r-1]})=\mathbb{P}(c_{ij}=1\mid z_{ik}=1;\theta^{[r-1]})=\rho(\alpha_{k})$$
Using \eqref{astuce_tauik}, one obtains
\begin{equation}
\label{eq:tauikMNARz_zj}
t_{ik}^{[r-1]}(\theta^{[r-1]}) = \frac{\pi^{[r-1]}_k \phi(\ty_i^{\mathrm{obs}};(\mu_{ik}^{\mathrm{obs}})^{[r-1]},(\Sigma_{ik}^{\mathrm{obs},\mathrm{obs}})^{[r-1]})\prod_{j=1}^d \rho(\alpha_{k}^{[r-1]})^{c_{ij}}(1-\rho(\alpha_{k}^{[r-1]}))^{1-c_{ij}}}{\sum_{h=1}^K \pi^{[r-1]}_h \phi(\ty_i^{\mathrm{obs}};(\mu_{ih}^{\mathrm{obs}})^{[r-1]},(\Sigma_{ih}^{\mathrm{obs},\mathrm{obs}})^{[r-1]})\prod_{j=1}^d \rho(\alpha_{h}^{[r-1]})^{c_{ij}}(1-\rho(\alpha_{h}^{[r-1]}))^{1-c_{ij}}}
\end{equation}
\end{itemize}
If $\rho$ is the logistic distribution, the expression can be written more simply
\begin{equation*}
t_{ik}(\theta^{[r-1]}) \propto \pi_k^{[r-1]} \phi(\ty_i^{\mathrm{obs}};\lambda_k^{[r-1]}) \prod_{j=1}^d \left(1+\exp(-\delta_{ij} \alpha_{k}^{[r-1]})\right)^{-1}
\text{ where }\delta_{ij}=\left\lbrace \begin{array}{rl}
1 &\text{ if } c_{ij}=1\\
-1 &\text{ otherwise.}
\end{array}\right.
\end{equation*}

Finally, the E-step and the M-step can be sketched as follows in the Gaussian mixture case. 
\medbreak
\textbf{E-step} 
The E-step for Term \ref{eiy} consists of computing for $k=1,\ldots,K$ and $i=1,\ldots,n$
\begin{equation*}
\begin{split}
(\tilde{\mu}_{ik}^{\mathrm{mis}})^{[r-1]}&= (\mu_{ik}^{\mathrm{mis}})^{[r-1]} + (\Sigma_{ik}^{\mathrm{mis},\mathrm{obs}})^{[r-1]} \left((\Sigma_{ik}^{\mathrm{obs},\mathrm{obs}})^{[r-1]}\right)^{-1} \left(\ty_i^{\mathrm{obs}} - (\mu_{ik}^{\mathrm{obs}})^{[r-1]} \right)\\
(\tilde{\Sigma}_{ik}^{\mathrm{mis}})^{[r-1]} &= (\Sigma_{ik}^{\mathrm{mis},\mathrm{mis}})^{[r-1]} - (\Sigma_{ik}^{\mathrm{mis},\mathrm{obs}})^{[r-1]} \left((\Sigma_{ik}^{\mathrm{obs},\mathrm{obs}})^{[r-1]}\right)^{-1} (\Sigma_{ik}^{\mathrm{obs},\mathrm{mis}})^{[r-1]}\\
(\tilde{y}_{i,k})^{[r-1]} &= (\ty_i^{\mathrm{obs}},(\tilde{\mu}_{ik}^{\mathrm{mis}})^{[r-1]})\\
\tilde{\Sigma}_{ik}^{[r-1]} &= \left( \begin{array}{cc}
0_i^{\mathrm{obs},\mathrm{obs}} & 0_i^{\mathrm{obs},\mathrm{mis}}\\
0_i^{\mathrm{mis},\mathrm{obs}} & (\tilde{\Sigma}_{ik}^{\mathrm{mis}})^{[r-1]}
\end{array}\right).
\end{split}
\end{equation*}
Note that whenever the mixture covariance matrices are supposed diagonal then $(\tilde{\Sigma}_{ik}^{\mathrm{mis}})^{[r-1]}$ is also a diagonal matrix. %This M step is equivalent to AAC EM algorithm.
Term \ref{tauik} also requires the computation of $t_{ik}(\theta^{[r-1]})$ given in \eqref{eq:tauikMNARz_zj} for $k=1,\ldots,K$ and $i=1,\ldots,n$.

\medbreak
\textbf{M-step} The maximization of $Q(\theta;\theta^{[r-1]})$ over $(\pi,\lambda)$ leads to, for $k=1,\ldots,K$,
\begin{equation*}
\begin{split}
\pi_k^{[r]} &= \frac{1}{n} \sum_{i=1}^n t_{ik}(\theta^{[r-1]}) \\
\mu_k^{[r]} &= \frac{\sum_{i=1}^n t_{ik}(\theta^{[r-1]}) (\tilde{y}_{k,i})^{[r-1]}}{\sum_{i=1}^n t_{ik}(\theta^{[r-1]})}\\
\Sigma_k^{[r]} &= \frac{\sum_{i=1}^n \left[t_{ik}(\theta^{[r-1]}) \left((\tilde{y}_{i,k})^{[r-1]} - \mu_k^{r})((\tilde{y}_{i,k})^{[r-1]} - \mu_k^{r})^T+\tilde{\Sigma}_{ik}^{[r-1]}\right) \right]}{\sum_{i=1}^n t_{ik}(\theta^{[r-1]})}.
\end{split}
\end{equation*}
Then, the maximization of $Q(\theta;\theta^{[r-1]})$ over $\psi$ can be performed using a Newton Raphson algorithm. For $k=1,\dots,K$, it remains to fit a generalized linear model with the binomial link function for the matrix $(\mathcal{J}_k^{\textrm{MNAR$z$}})^{[r]}$ and by giving $t_{ik}(\theta^{[r-1]})$ as prior weights to fit the process.
\begin{equation}
\label{eq:phi_JMNARz}
(\mathcal{J}_k^{\textrm{MNAR$z$}})^{[r]}=\left(\begin{array}{cc} c_{.1} & 1
\\
\vdots & \vdots \\
c_{.d} & 1\end{array}\right).
\end{equation}

The EM algorithm for the MNAR$\tz$ model is described in Algorithm \ref{alg:EM} for Gaussian mixture.

\begin{algorithm}
\caption{EM algorithm for Gaussian mixture and MNAR$z$ model}
\label{alg:EM}
\begin{algorithmic}[1]
\STATE {\bfseries Input:} $Y \in \mathbb{R}^{n\times d}$ (matrix containing missing values), $K\geq 1$, $r_\mathrm{max}$.
\STATE Initialize $\pi_k^0$, $\mu_{k}^0,\Sigma_k^0$ and $\psi_k^0$, for $k \in \{1,\dots,K\}$.
\FOR{$r=0$ {\bfseries to} $r_\mathrm{max}$}
    
    \STATE \textbf{E-step}:
    
    \FOR{$i=1$ {\bfseries to} $n$, $k=1$ {\bfseries to} $K$}
        \STATE $(\tilde{\mu}_{ik}^{\mathrm{mis}})^{[r-1]}= (\mu_{ik}^{\mathrm{mis}})^{[r-1]} + (\Sigma_{ik}^{\mathrm{mis},\mathrm{obs}})^{[r-1]} \left((\Sigma_{ik}^{\mathrm{obs},\mathrm{obs}})^{[r-1]}\right)^{-1} \left(\ty_i^{\mathrm{obs}} - (\mu_{ik}^{\mathrm{obs}})^{[r-1]} \right)$.
        \STATE $(\tilde{\Sigma}_{ik}^{\mathrm{mis}})^{[r-1]} = (\Sigma_{ik}^{\mathrm{mis},\mathrm{mis}})^{[r-1]} - (\Sigma_{ik}^{\mathrm{mis},\mathrm{obs}})^{[r-1]} \left((\Sigma_{ik}^{\mathrm{obs},\mathrm{obs}})^{[r-1]}\right)^{-1} (\Sigma_{ik}^{\mathrm{obs},\mathrm{mis}})^{[r-1]}$.
        \STATE $(\tilde{y}_{i,k})^{[r-1]} = (\ty_i^{\mathrm{obs}},(\tilde{\mu}_{ik}^{\mathrm{mis}})^{[r-1]})$.
        \STATE $\tilde{\Sigma}_{ik}^{[r-1]} = \left( \begin{array}{cc}
0_i^{\mathrm{obs},\mathrm{obs}} & 0_i^{\mathrm{obs},\mathrm{mis}}\\
0_i^{\mathrm{obs},\mathrm{mis}} & (\tilde{\Sigma}_{ik}^{\mathrm{mis}})^{[r-1]}
\end{array}\right)$, where $0_i^{\mathrm{obs},\mathrm{obs}}$ and $0_i^{\mathrm{obs},\mathrm{mis}}$ are the null matrix of size $n_i^\mathrm{obs}\times n_i^\mathrm{obs}$ and $n_i^\mathrm{obs}\times n_i^\mathrm{mis}$, with $n_i^\mathrm{obs}$ (resp. $n_i^\mathrm{miss}$) the number of observed (reps. missing) variables for individual $i$.
        \STATE $t_{ik}(\theta^{[r-1]}) \propto
\pi^{[r-1]}_k \phi(\ty_i^{\mathrm{obs}};(\mu_{ik}^{\mathrm{obs}})^{[r-1]},(\Sigma_{ik}^{\mathrm{obs},\mathrm{obs}})^{[r-1]})\prod_{j=1}^d \rho(\alpha_{k}^{[r-1]})^{c_{ij}}(1-\rho(\alpha_{k}^{[r-1]}))^{1-c_{ij}}$
    
    \ENDFOR
    
    \STATE \textbf{M-step}:
    
    \FOR{$k=1$ {\bfseries to} $K$}
    \STATE $\pi_k^{[r]} = \frac{1}{n} \sum_{i=1}^n t_{ik}(\theta^{[r-1]}), \qquad
\mu_k^{[r]} = \frac{\sum_{i=1}^n t_{ik}(\theta^{[r-1]}) (\tilde{y}_{k,i})^{[r-1]}}{\sum_{i=1}^n t_{ik}(\theta^{[r-1]})}$
    \STATE $\Sigma_k^{[r]} = \frac{\sum_{i=1}^n \left[t_{ik}(\theta^{[r-1]}) \left((\tilde{y}_{i,k})^{[r-1]} - \mu_k^{[r]})((\tilde{y}_{i,k})^{[r-1]} - \mu_k^{[r]})^T+\tilde{\Sigma}_{ik}^{[r-1]}\right) \right]}{\sum_{i=1}^n t_{ik}(\theta^{[r-1]})}$
    \STATE Let $\psi_k^{[r]}$ be the coefficients of a GLM with a binomial link function, by giving prior weights $t_{ik}(\theta^{[r-1]})$. In particular, the optimization problem is
    $$(\mathcal{J}_k^{\textrm{MNAR$z$}})^{[r]}\psi_k^{[r]} =\log\left(\frac{1-\mathbb{E}[\tc|(\mathcal{J}_k^{\textrm{MNAR$z$}})^{[r]}])}{\mathbb{E}[\tc|(\mathcal{J}_k^{\textrm{MNAR$z$}})^{[r]}]}\right),$$
    for a matrix $(\mathcal{J}_k^{\textrm{MNAR$z$}})^{[r]}$ given in \eqref{eq:phi_JMNARz} and $\tc=(\tc_{.1},\dots,\tc_{.d})$ the concatenated missing data patterns for the variables $1,\dots,d$. 
    
    \ENDFOR

\ENDFOR
\end{algorithmic}
\end{algorithm}

%\paragraph{MNAR$\tyall$ models}

\subsection{Latent class model for categorical data}
\label{app:EMcategorical}

%\subsubsection{Computations of $Q_{\ty,\tz}$}
%I will develop the term $Q_{\ty,\tz}$:
For categorical data, we have $\phi(\ty_i;\lambda_k)=\prod_{j=1}^d \phi(y_{ij};\lambda_{kj})=\prod_{j=1}^d \prod_{\ell=1}^{\ell_j} (\lambda_{kj}^{\ell})^{y_{ij}^\ell}$.

Term \ref{eiy} is
\begin{equation}
\mathbb{E}\left[\log(\phi(\ty_i;p_k))\mid \ty_i^{\mathrm{obs}},z_{ik}=1,\tc_i;\lambda^{[r-1]}\right] = \sum_{j, c_{ij}=0}\sum_{\ell=1}^{\ell_j} y_{ij}^\ell + \sum_{j, c_{ij}=1}\sum_{\ell=1}^{\ell_j} \log(\lambda_{kj}^{y_{ij}^\ell})
\label{Estep::logY_i_Mult}
\end{equation}
Term \ref{eic} is the same as in the Gaussian case given in \eqref{eq:termb_gauss_MNARz}. Finally, the EM algorithm can be summarized as follows

\medbreak
\textbf{E step:} For $k =1,\ldots,K$ and $i=1,\ldots,n$, compute
\begin{align*}
t_{ik}(\theta^{[r-1]})&=\frac{\pi^{[r-1]}_k \prod_{j, c_{ij}=0} \prod_{\ell=1}^{\ell_j}(\lambda_{kj}^\ell)^{y_{ij}^\ell}\rho(\alpha_{k})}{\sum_{h=1}^K \pi^{[r-1]}_h \prod_{j, c_{ij}=0} \prod_{\ell=1}^{\ell_j}(\lambda_{hj}^\ell)^{y_{ij}^\ell}\rho(\alpha_{h})} \\
(\tilde{y}_{ij,k}^\ell)^{[r-1]} &= c_{ij}(\theta_{kj}^\ell)^{[r-1]}+(1-c_{ij})y_{ij}^\ell, \: \forall j=1,\dots,d, \forall \ell=1,\dots,\ell_j.
\end{align*}

\medbreak
\textbf{M step:} The maximization of  $Q(\theta;\theta^{[r-1]})$ over $\theta$ leads to, for $k=1,\ldots,K$, 
\begin{equation*}
\begin{split}
\pi^{r}_k &= \frac{1}{n} \sum_{i=1}^n t_{ik}(\theta^{[r-1]})\\
(\theta_{kj}^\ell)^{r} &= \frac{\sum_{i=1}^n t_{ik}(\theta^{[r-1]}) (\tilde{y}_{ij,k}^\ell)^{[r-1]}}{\sum_{i=1}^n t_{ik}(\theta^{[r-1]})}, \: \forall j=1,\dots,d, \forall \ell=1,\dots,\ell_j.
\end{split}
\end{equation*}
The M-step for $\psi$ consists of performing a GLM with a binomial link for the following matrix: 
    \begin{equation}
    \label{eq:phi_HMNARz}
    (\mathcal{H}^{\textrm{MNAR$z$}})^{[r]}=\left(\begin{array}{c|ccc} c_{.1} & z_{.1} & \hdots & z_{.K}
    \\
    \vdots & \vdots & \vdots & \vdots \\
    c_{.d} & z_{.1} & \hdots & z_{.K}\end{array}\right)=\left(\begin{array}{c|ccc} c_{11} & z_{11}^{[r]} & \hdots & z_{1K}^{[r]}  \\
    \vdots & \vdots & \vdots & \vdots \\
    c_{n1} & z_{n1}^{[r]} & \hdots & z_{nK}^{[r]} \\
    \vdots & \vdots & \vdots & \vdots\\
    c_{1d} & z_{11}^{[r]} & \hdots & z_{1K}^{[r]}\\
    \vdots & \vdots & \vdots & \vdots\\
    c_{nd} & z_{n1}^{[r]} & \hdots & z_{nK}^{[r]}
    \end{array}.\right)
    \end{equation}

\subsection{Combining Gaussian mixture and latent class model for mixed data}

If the data are mixed (continuous and categorical), the formulas can be extended straightforwardly if the continuous and the categorical variables are assumed to be independent knowing the latent clusters. 
%And, the E-step encounters the same difficulty that it has in the continuous case. %Moreover, in this situation, it is recommended to assumed that the mixture covariance matrices are diagonal for a fair treatment between the continuous and the categorical variables. 
%then one has to use within each variable the needed EM step considering that there is no correlation between variables $\Longrightarrow$ AAC algorithm. 

\section{Additional numerical experiments on synthetic data}\label{sec:app_xp}

Note that in Figure \ref{fig:xp1_MNARz}, the differences for $n=100$ can be explained by the difference in initialization of the algorithms, which can play an important role for small sample sizes.  

\begin{figure}
\centering
\includegraphics[scale=0.8]{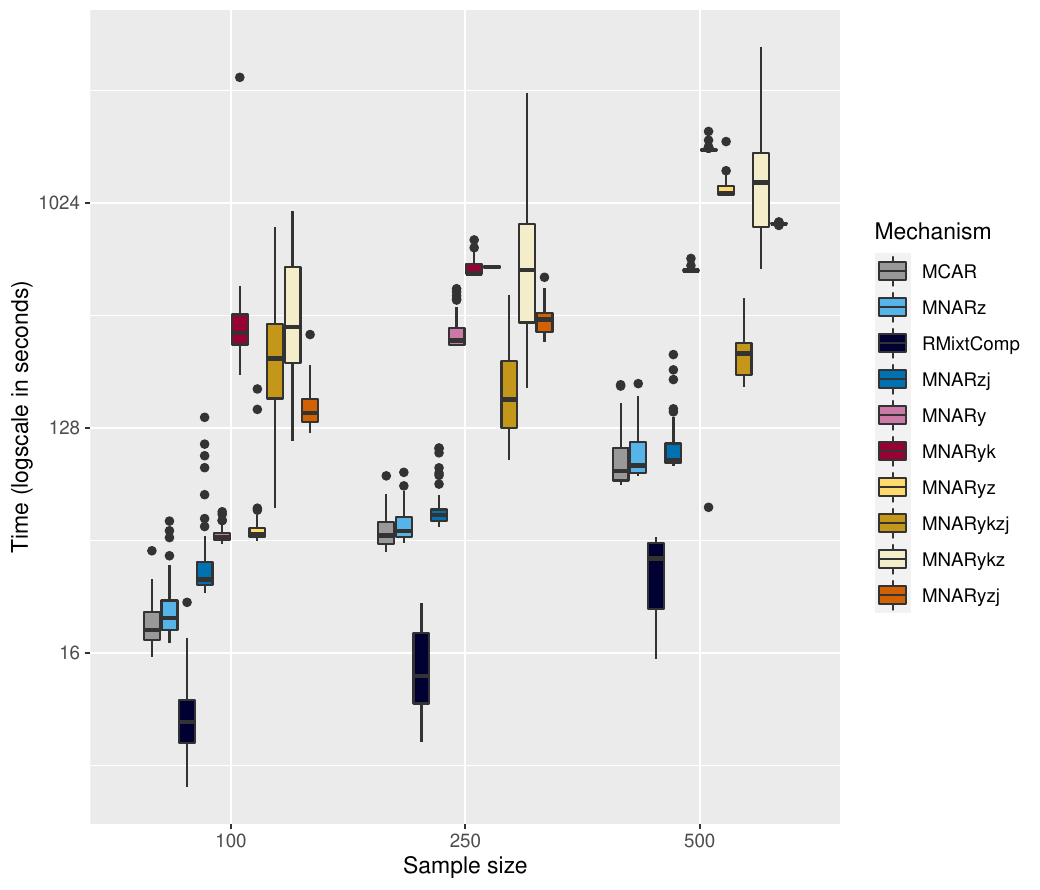}
\caption{\label{fig:comptime} Boxplot of the computational times (in seconds) obtained for 50 samples composed of $d = 6$ variables.} %with a misclassification rate of $10\%$ and a missing-data rate of $30\%$ in the whole dataset (see experiment on the consistency of the estimators illustrated by Figure \ref{fig:xp1}).}
\end{figure}

\begin{figure}
    \centering
    \includegraphics[scale=0.7]{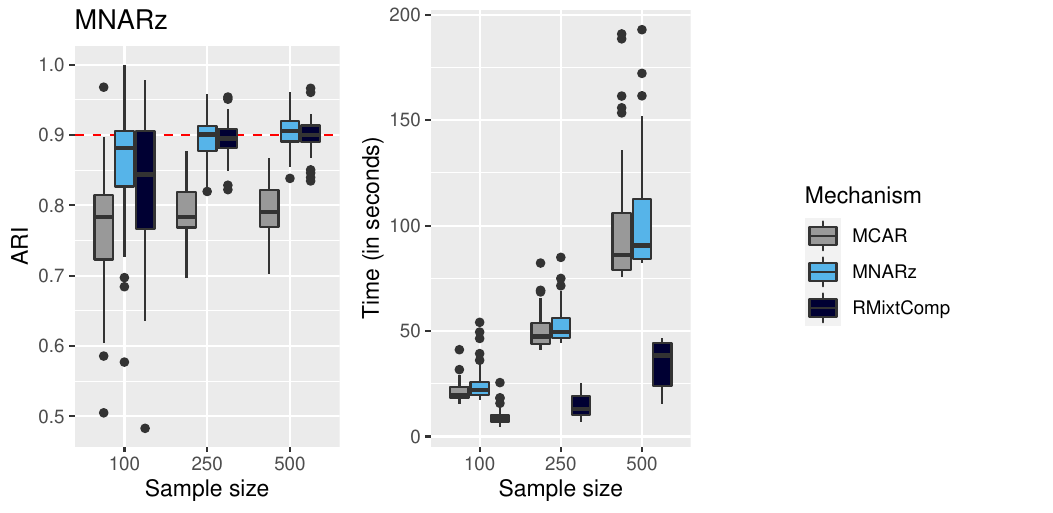}
    \caption{Left graphic: Boxplot of the ARI obtained for 50 samples composed of $d = 6$ variables and $n\in \{100,250,500\}$. In grey: our EM implementation for MCAR data, in blue: our EM implementation for MNAR$z$ data, in green: SEM algorithm of \textbf{RMixtComp} for MCAR data using the augmented data matrix.
    Right graphic: associated computational times (in seconds).}
    \label{fig:xp1_MNARz}
\end{figure}

\begin{figure}
     \centering
     \includegraphics[scale=0.5]{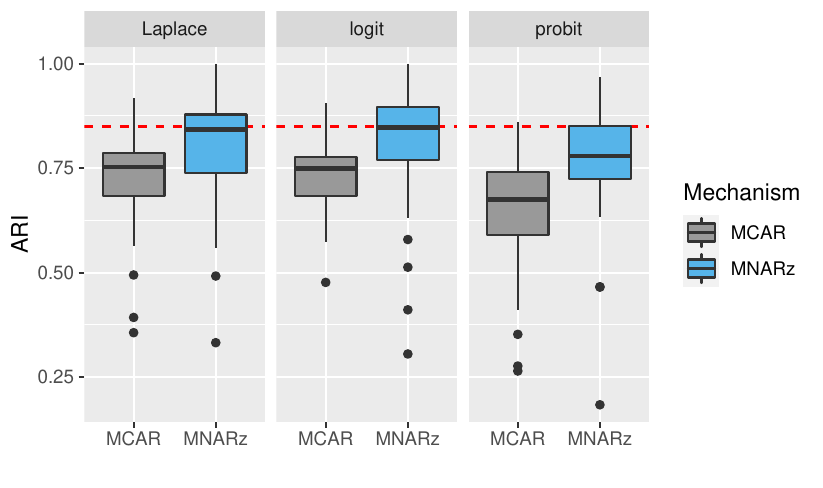}
     \caption{Boxplot of the ARI obtained for 50 samples composed of $d = 6$ variables. The missing values are introduced using a MNAR$z$ setting. The misclassification rate is of $15\%$. Impact of the misspecification of the link function.}
     \label{fig:xp3a_app}
\end{figure}
\begin{figure}
         \centering
         \includegraphics[scale=0.45]{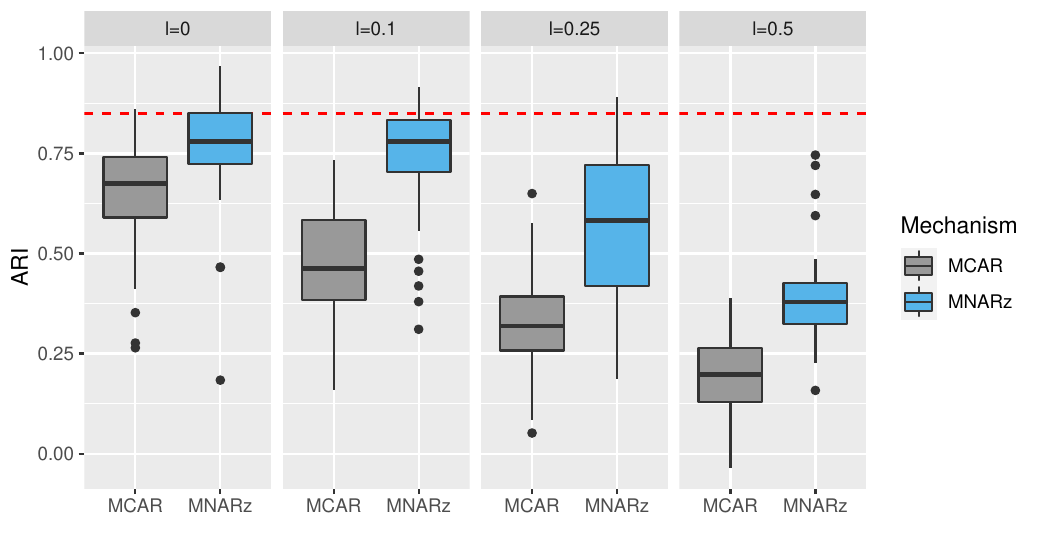}
         \caption{Boxplot of the ARI obtained for 50 samples composed of $d = 6$ variables. The missing values are introduced using a MNAR$z$ setting. The misclassification rate is of $15\%$. Impact of the misspecification of the data distribution.}
         \label{fig:xp3b_app}
\end{figure}
\begin{figure}
         \centering
         \includegraphics[scale=0.42]{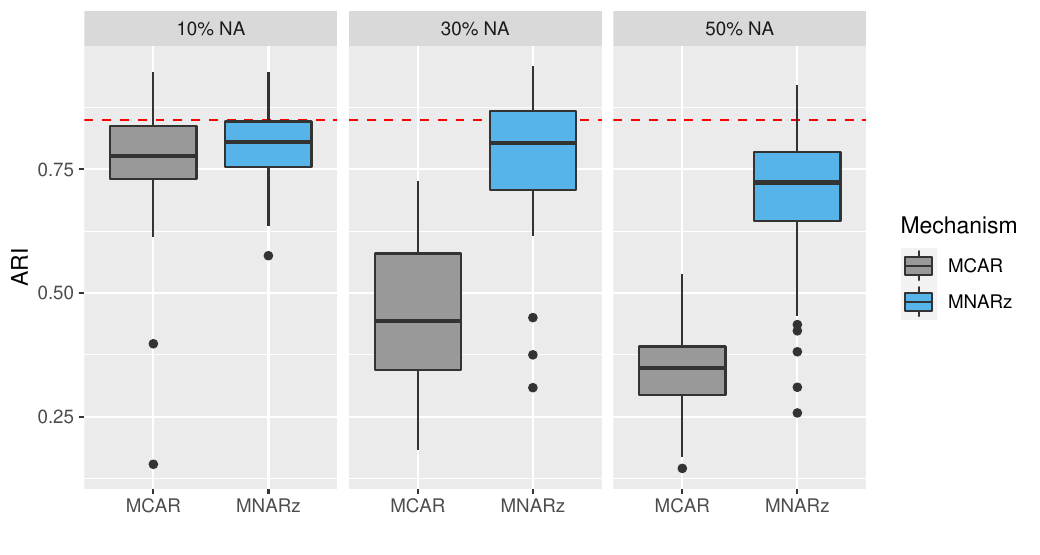}
         \caption{Boxplot of the ARI obtained for 50 samples composed of $d = 6$ variables. The missing values are introduced using a MNAR$z$ setting. The misclassification rate is of $15\%$. Impact of the percentage of missing values.}
         \label{fig:xp4a_app}
\end{figure}

\section{Complements on generic experiments}
\label{app:tab}

This section gives the values of $\delta$ (see \eqref{eq:simu}) $\psi$ (see \eqref{eq:meca}) and $\varphi$ (see \eqref{eq:simu}) used during the different experiments. As explained in Section \ref{sec:subsec:genericXP}, their choice allows to control the rates of misclassification
and missingness, as well as the interation between the variables and the class membership. To estimate these values, we have generated a large sample ($n=10^5$) and compute the misclassification rate and the missingness rate for several values of $\delta$ and $\psi$ and pick the ones which correspond to the setting of the experiment.

\begin{center}
\begin{table}
\begin{tabular}{cc}
\hline

  $d$  & $\varphi$    \\ \hline
  3 & $\varphi_{11}=\varphi_{22}=\varphi_{33}=1$\\
  \hline   
  6 & $\varphi_{11}=\varphi_{22}=\varphi_{33}=\varphi_{14}=\varphi_{36}=1$\\
  \hline  
  9 &  $\varphi_{11}=\varphi_{22}=\varphi_{33}=\varphi_{14}=\varphi_{36}=\varphi_{17}=\varphi_{27}=\varphi_{39}=1$
\end{tabular}
\caption{Choice of the values of $\varphi$ and $\alpha$ for all the experiments of Section \ref{sec:subsec:genericXP}. Other values $\varphi_{kj}$ are null.}
\end{table}
\end{center}

\begin{center}
\begin{table}
\begin{tabular}{ccccccc}
\hline
  $K$  & \% NA & link & rate of misclassification & $l$ & $\delta$  & $\alpha$  \\ \hline
3 & 30\% & probit & 90\% & 0 & 2.6
& $\begin{pmatrix}-1 &-0.3 & 0\end{pmatrix}$ \\
 \hline
 3 & 30\% & logit & 90\% & 0 & 2.76 & $\begin{pmatrix}-1.5 & -0.8 & 0.1\end{pmatrix}$ \\
\hline
 3 & 30\% & Laplace & 90\% & 0 & 2.85 & $\begin{pmatrix}-1.1 & 0.3 & 0\end{pmatrix}$
\\
 \hline
3 & 30\% & probit & 85\% & 0 &  2.27 & $\begin{pmatrix}-1 & -0.3 & 0\end{pmatrix}$ \\
 \hline
3 & 30\% & logit & 85\% & 0 & 2.44 & $\begin{pmatrix}-1.5 & -0.8 & 0.1\end{pmatrix}$ \\
\hline
3 & 30\% & Laplace & 85\% & 0 & 2.46 & $\begin{pmatrix}-1.1 & 0.3 & 0\end{pmatrix}$ \\
\hline
3 & 30\% & probit & 90\% & 0.1 & 2.3 & $\begin{pmatrix}-1.16 & 0.3 & -0.42 \end{pmatrix}$
\\
\hline
3 & 30\% & probit & 90\% & 0.25 & 2.17 & $\begin{pmatrix}-1.16 & 0.3 & -0.4 \end{pmatrix}$
\\
\hline
3 & 30\% & probit & 90\% & 0.5 & 1.85 & $\begin{pmatrix}-1.16 & 0.3 & -0.4 \end{pmatrix}$ \\
\hline
3 & 30\% & probit & 85\% & 0.1 & 1.97 & $\begin{pmatrix}-1.16 & 0.3 & -0.42 \end{pmatrix}$
\\
\hline
3 & 30\% & probit & 85\% & 0.25 & 1.86 & $\begin{pmatrix}-1.16 & 0.3 & -0.4 \end{pmatrix}$
\\
\hline
3 & 30\% & probit & 85\% & 0.5 & 1.57 & $\begin{pmatrix}-1.16 & 0.3 & -0.4 \end{pmatrix}$ \\
\hline
3 & 10\% & probit & 90\% & 0 & 2.18
& $\begin{pmatrix}-1.65 &-1.2 & -0.9\end{pmatrix}$ \\
\hline
3 & 50\% & probit & 90\% & 0 & 3.3
& $\begin{pmatrix}-0.55 &0.25& 1.7\end{pmatrix}$ \\
 \hline
 3 & 10\% & probit & 85\% & 0 & 1.95
& $\begin{pmatrix}-1.65 &-1.2 & -0.9\end{pmatrix}$ \\
\hline
3 & 50\% & probit & 85\% & 0 & 2.62
& $\begin{pmatrix}-0.55 &0.25& 1.7\end{pmatrix}$ \\
 \hline
\end{tabular}
\caption{Choice of the values of $\delta$ and $\alpha$ for all the experiments of Section \ref{sec:subsec:genericXP} and Appendix \ref{sec:app_xp} for the MNAR$z$ mechanism. $K$ denotes the number of class, the column denoted as \% NA gives the rate of missingness, the column called link gives the link function of the missing-data mechanism used in the introduction of the missing values, $l$ is the coefficient of correlation (anti-diagonal terms), $\delta$ is given in \eqref{eq:simu} and $\alpha$ in \eqref{eq:meca}.}
\end{table}
\end{center}

\begin{center}
\begin{table}
\begin{tabular}{ccc}
\hline

  $d$  & $\delta$  & $\alpha$  \\ \hline
  3 & 20 & $\begin{pmatrix} -0.4 & -0.65 & -0.65 \\ -1.1 & -1 & -1 \\ -0.6 & 0.4 & 0.4
  \end{pmatrix}$ \\
  \hline   
  6 & 2.5 & $\begin{pmatrix} -1.4 & -1.4 & -1.2 & -1.1 & -1 & -0.9 \\
  -0.6 & 0.4 & 0.4 & 0.3 & 0.1 & 0.1 \\
  -0.2 & -0.2 & -0.2 & -0.2 & -0.2 & -0.2
  \end{pmatrix}$ \\
  \hline  
  9 & 1.78 & $\begin{pmatrix} -0.5 & -0.65 & -0.65 & -1.1 & -1.7 & -1.7 & -1.4 & -1.4 & -1.4 \\
  -0.6 & 0.4 & 0.4 & -0.2 & 0.3 & 0.4 & 0.3 & 0.3 & 0.3 \\
  -0.4 & -0.4 &-0.4 &-0.4 &-0.4 &-0.4 &-0.4 &-0.4 & -0.4 
  \end{pmatrix}$
\end{tabular}
\caption{Choice of the values of $\delta$ and $\alpha$ for all the experiments of Section \ref{sec:subsec:genericXP} for the MNAR$\tzj$ mechanism.}
\end{table}
\end{center}

\begin{center}
\begin{table}[]
\begin{tabular}{cccc}
\hline
  $d$  & $\delta$  & $\alpha$  & $\beta$ \\ \hline
  3 & 3.5 & -1.56 & $\begin{pmatrix} 1.45 & 0.2 & -3
  \end{pmatrix}$ \\
  \hline  
  6 & 2.25  & -0.7 & $\begin{pmatrix} -3 & 0.3 & -3& -3& -2 & 1
  \end{pmatrix}$ \\
  \hline  
  9 & 1.98 &  -0.68
 & $\begin{pmatrix} 0.5 & 0.1 & -1.2 & 0.4 & -0.1 & -1.3 & 0.3 & -0.1 & -1
  \end{pmatrix}$
\end{tabular}
\caption{Choice of the values of $\delta$, $\alpha$ and $\beta$ for all the experiments of Section \ref{sec:subsec:genericXP} for the MNAR$y$ mechanism.}
\end{table}
\end{center}

\begin{center}
\begin{table}[]
\begin{tabular}{cccc}
\hline
  $d$  & $\delta$  & $\alpha$  & $\beta$ \\ \hline
    3 & 4.72 & $\begin{pmatrix}-1.2 &-0.8 &-0.5
  \end{pmatrix}$& $\begin{pmatrix} -3&0.3&1
  \end{pmatrix}$ \\
  \hline 
  6 & 2.12
  & $\begin{pmatrix} -1.35 & -0.29 & 0
  \end{pmatrix}$ & $\begin{pmatrix} -3& 0.3 & -3 &-3 & -2& 1
  \end{pmatrix}$ \\
  \hline  
  9 & 1.71
 &  $\begin{pmatrix}-1.34&-0.34&0
  \end{pmatrix}$
 & $\begin{pmatrix} -3& 0.3&-3&-2.8&-2&1&0.2&0.1&0.4
  \end{pmatrix}$
\end{tabular}
\caption{Choice of the values of $\delta$, $\alpha$ and $\beta$ for all the experiments of Section \ref{sec:subsec:genericXP} for the MNAR$yz$ mechanism.}
\end{table}
\end{center}

\begin{center}
\begin{table}[]
\begin{tabular}{cccc}
\hline
  $d$  & $\delta$  & $\alpha$  & $\beta$ \\ \hline
    3 & 2.55 & {\footnotesize $\begin{pmatrix}-1 &-0.95 &-0.9 \\
    0.75&0.7&0.8 \\
    -0.2 & -0.2 & -0.2
  \end{pmatrix}$} & {\footnotesize$\begin{pmatrix} -3 & 0.3 & -3
  \\
 0.3 & -3 & 0.3 \\
 -3 & 0.3 & -3
  \end{pmatrix}$} \\
  \hline 
  6 & 1.96
  & {\footnotesize $\begin{pmatrix} -1.2 & -1 & -0.9 & -0.9 & -0.7 & -0.8 \\
  -0.6 & 0.4 & 0.4 & 0.3 & 0.1 & 0.1 \\
  -0.4 & -0.4 & -0.4 & -0.4 & -0.4 & -0.4 
  \end{pmatrix}$} & {\footnotesize $\begin{pmatrix} -3 & 0.3 & -3 & -3 & -2 & 1 \\
  0.3 & -3 & 0.3 & -0.3 & -2 & 0.2 \\
  -3 & 0.3 & -3 & -3 & -2 & 1
  \end{pmatrix}$} \\
  \hline  
  9 & 1.45
 &  {\scriptsize $\begin{pmatrix}-1.4 & -1 & -1.1 & -1.1 & -0.9 & -0.8 & -1.2 & -1 & -1.1 \\
 0.3 & 0.5 & 0.2 & -0.6 & 0.4 & 0.4 & 0.3 & 0.1 & 0.1 \\
 -0.4 & -0.4 & -0.4 & -0.4 & -0.4 & -0.4 & -0.4 & -0.4 & -0.4
  \end{pmatrix}$}
 & {\scriptsize $\begin{pmatrix} -3& 0.3&-3&-3&-2&1&-3&0.3&0.2 \\
 0.3 & -3 & 0.3 & -0.3 & -2 & 0.2 & 0.2 & 0.3 & -0.3 \\
 -3 & 0.3 & -3 & -3 & -2 & 1 & -1 & -2 & -3
  \end{pmatrix}$}
\end{tabular}
\caption{Choice of the values of $\delta$, $\alpha$ and $\beta$ for all the experiments of Section \ref{sec:subsec:genericXP} for the MNAR$\tyk\tzj$ mechanism.}
\end{table}
\end{center}

\section{Traumabase dataset}\label{app:realdata}

\subsection{Impact of the MNAR$z$ process on the estimated partition}

Table \ref{tab:realdata_probaclass} gives the Euclidean distance between the conditional probabilities of the cluster memberships given the observed values of the variable \textit{Shock.index.ph} obtained using the algorithm considering MNAR$z$ data and those obtained using the algorithm considering MCAR data. For clarity, the latter quantity is reported here,
$$ \sqrt{\sum_{i=1}^n(\mathbb{P}(z_{ik}=1|y_{is}^\mathrm{obs};\theta^\mathrm{MCAR})-\mathbb{P}(z_{i\tilde{k}}=1|y_{is}^\mathrm{obs};\theta^\mathrm{MNAR}))^2}, \forall k,\tilde{k} \in \{1,2,3\}$$
with $s$ the index of the variable \textit{Shock.index.ph}, $\theta^\mathrm{MCAR}$ (resp. $\theta^\mathrm{MNAR}$)  the estimator returned by the algorithm considering MCAR data (resp. MNAR data). 

\newpage

\begin{landscape}
\begin{figure}
%\rotatebox{90}{
\includegraphics[scale=0.65]{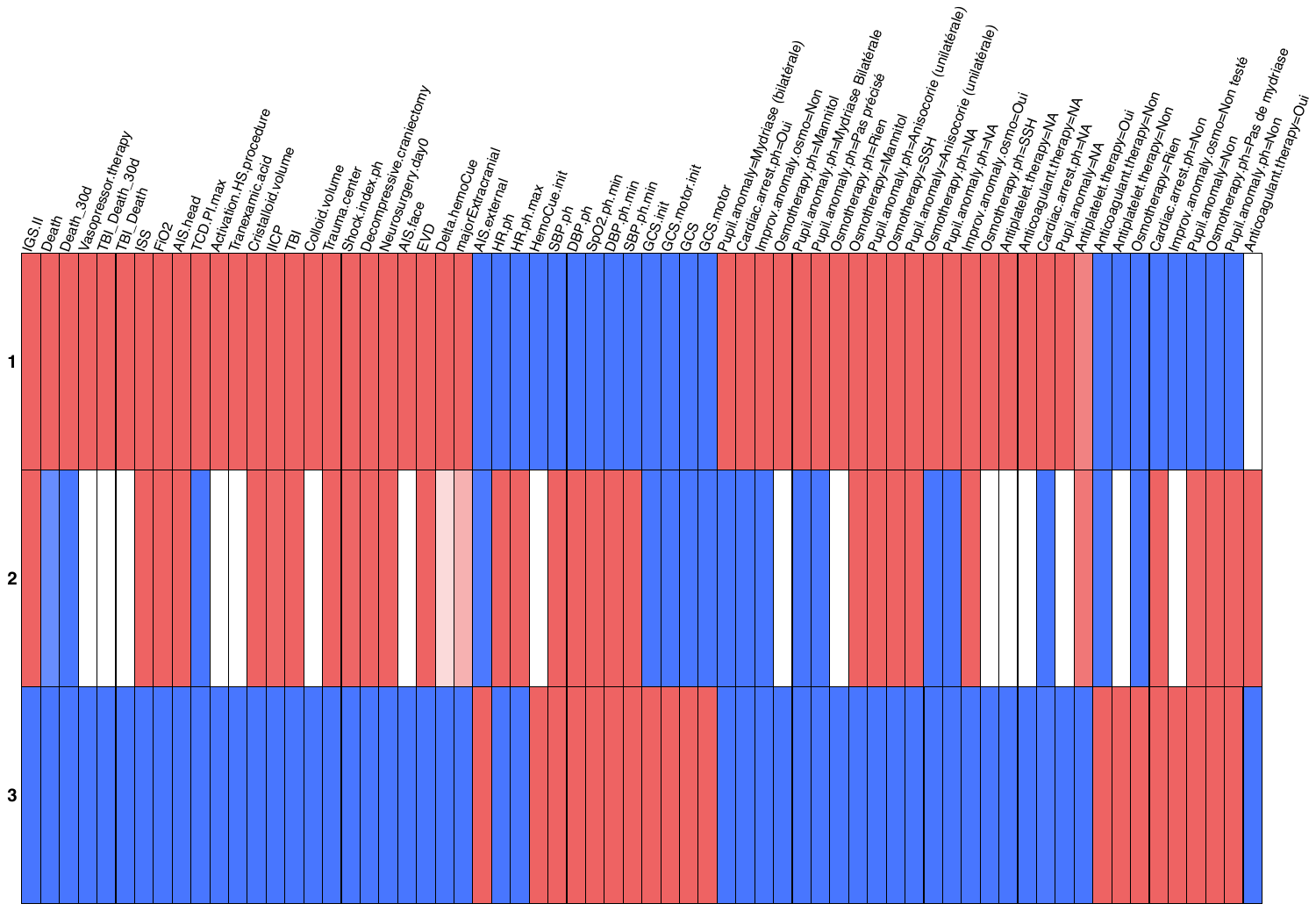}
%}
\caption{Plot of the catdes output for the EM algorithm applied to the Traumabase dataset considering MNAR$z$ data. We consider $K= 3$ clusters. A cell colored in blue (resp. red) means that the average value of the variable in the given cluster is significantly lower (resp. higher) than in the overall data. In addition, the more transparent the cell is, the less significant the difference between the behavior of the variable in the given cluster and in the overall data is.}
\label{fig:MNARz_41var}
\end{figure}
\end{landscape}

\subsection{Imputation performances in the Traumabase dataset}

We perform now simulations on the real dataset in order to be able to measure the quality of the imputation of our method compared to the multiple imputation \citep{buuren2010mice} (Mice). %\jj{c'est bizare si on veut regarder la qualité de l'imputation de se comparer à l'imputation multiple non?}. 
We introduce some additional missing values in three quantitative variables (\textit{TCD.PI.max, Shock.index.ph, FiO2}) by using the MNAR$z$ mechanism \eqref{eq:MNARz}. %\as{c'est mieux de faire ICE ?} 
The variables contain initially 51\%, 31\%, 7\% and finally 63\%, 50\% and 32\% missing values. The algorithm for continuous data specifically designed for MNAR$z$ data for $K=3$ classes is compared with mean imputation and multiple imputation in terms of mean squared error (MSE). Denoting by $\hat{Y} \in \mathbb{R}^{n\times d}$ the imputed dataset and $\tilde{C}\in \mathbb{R}^{n\times d}$ the indicator pattern of missing data newly introduced, the mean squared error is given by
$$\mathbb{E}[(\hat{Y}-Y)\odot\tilde{C}]_F^2\big/\mathbb{E}[Y\odot \tilde{C}]_F^2,$$
where $\odot$ is the Hadamard product and $\mathbb{E}[]_F^2=\mathbb{E}[\|.\|_F^2]$ denotes the expectation of the Frobenius norm squared.
In particular, to impute missing values using our clustering algorithm, we use the conditional expectation of the missing values given the observed ones, given that the data are assumed to be Gaussian and that all the parameters of the distribution are given by our algorithm. Imputation is carried out by taking the mean over $10^4$ draws. 
In Figure \ref{fig:MNARz_imp}, our clustering algorithm, designed for the MNAR setting, gives a significantly smaller error than other methods.

\begin{figure}[H]
\centering
\includegraphics[scale=0.6]{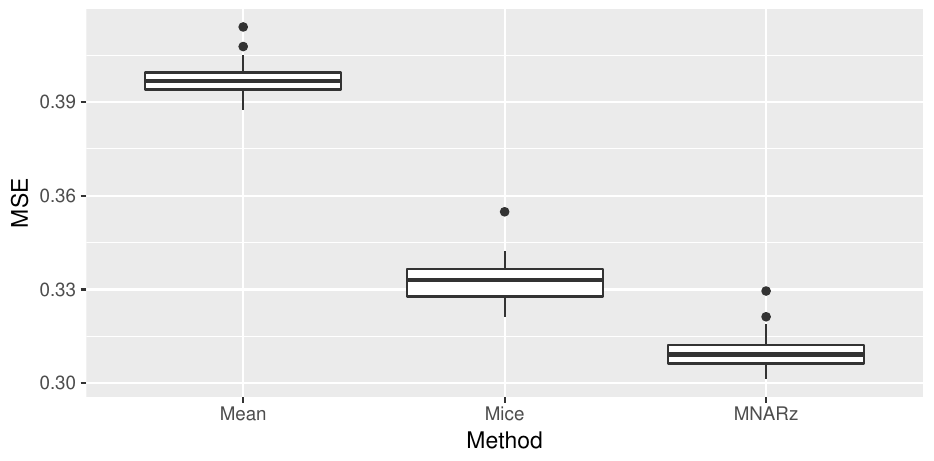}
\caption{Mean squared error of the imputation task for the Traumabase dataset.}
\label{fig:MNARz_imp}
\end{figure}

\subsection{Description of the variables in the Traumabase dataset}

A description of the variables which are used in Section \ref{sec_realdata} is given. Figure \ref{fig:percNA} gives the percentage of missing values per variable. The indications given in parentheses ph (pre-hospital) and h (hospital) mean that the measures have been taken before the arrival at the hospital and at the hospital.

\begin{itemize}
    \item \textit{Trauma.center} (categorical, integers between 1 and 16, no missing values): name of the trauma center (ph \& h).
    \item \textit{Anticoagulant.therapy} (categorical, binary variable, 4.3\% NA): oral anticoagulant therapy before the accident (ph). 
    \item \textit{Antiplatelet.therapy} (categorical, binary variable, 4.4\% NA): anti-platelet therapy before the accident (ph).
    \item \textit{GCS.init, GCS} (ordinal, integers between 3 and 15, 2\% NA \& 42\% NA): Initial Glasgow Coma Scale (GCS) on arrival on scene of enhanced care team and on arrival at the hospital (GCS = 3: deep coma; GCS = 15: conscious and alert) (ph \& h). 
    \item \textit{GCS.motor.init, GCS.motor} (ordinal, integers between 1 and 6, 7.6\% NA \& 43\%): Initial Glasgow Coma Scale motor score (GCS.motor = 1: no response; GCS.motor = 6: obeys command/purposeful movement) (ph \% h). 
    \item \textit{Pupil.anomaly.ph, Pupil.anomaly} (categorical, 3 categories: Non, Anisocoire (unilaterale), Mydriase Bilaterale, 2\% NA \& 1.7\%): pupil dilation indicating brain herniation (ph \& h).
    \item \textit{Osmotherapy.ph, Osmotherapy} (categorical, 4 categories: Pas de mydriase, SSH, Mannitol, Rien, 1.7\% NA and no missing values): administration of osmotherapy to alleviate compression of the brain (either Mannitol or hypertonic saline solution) (ph \& h)
    \item \textit{Improv.anomaly.osmo} (categorical, 3 categories: Non testé, Non, Oui, no missing values): change of pupil anomaly after ad- ministration of osmotherapy (ph).
    \item \textit{Cardiac.arrest.ph} (categorical, binary variable, 2.3\% NA): cardiac arrest during pre-hospital phase (ph).
    \item \textit{SBP.ph, DBP.ph, HR.ph} (continuous, 29.3\% NA \& 29.6\% NA \& 29.5\% NA): systolic and diastolic arterial pressure and heart rate during pre-hospital phase (ph).
    \item \textit{SBP.ph.min, DBP.ph.min} (continuous, 12.8\% NA \& 13\% NA): minimal systolic and diastolic arterial pressure during pre-hospital phase (ph).
    \item \textit{HR.ph.max} (continuous, 13.7 \% NA): maximal heart rate during pre-hospital phase (ph).
    \item \textit{Cristalloid.volume} (continuous, positive values, 30\% NA): total amount of prehospital adminis- tered cristalloid fluid resuscitation (volume expansion) (ph).
    \item \textit{Colloid.volume} (continuous, positive values, 31.3\% NA): total amount of prehospital administered colloid fluid resuscitation (volume expansion) (ph).
    \item \textit{HemoCue.init} (continuous, 34.9\% NA): prehospital capillary hemoglobin concentration (the lower, the more the patient is probably bleeding and in shock); hemoglobin is an oxygen carrier molecule in the blood (ph).
    \item \textit{Delta.hemoCue} (continuous, 37.2\% NA): difference of hemoglobin level between arrival at the hospital and arrival on the scene (h).
    \item \textit{Vasopressor.therapy} (continuous, no missing values): treatment with catecholamines in case of physical or emotional stress increasing heart rate, blood pres- sure, breathing rate, muscle strength and mental alertness (ph).
    \item \textit{SpO2.min} (continuous, 11.7\% NA): peripheral oxygen saturation, measured by pulse oxymetry, to estimate oxygen content in the blood (95 to 100\%: considered normal; inferior to 90\% critical and associated with considerable trauma, danger and mortality) (ph).
    \item \textit{TCD.PI.max} (continuous, 51.2\% NA): pulsatility index (PI) measured by echodoppler sonographic examen of blood velocity in cerebral arteries (PI > 1.2: indicates altered blood flow maybe due to traumatic brain injury) (h).
    \item \textit{FiO2} (categorical, in $\{1,2,3,4,5\}$, 6.8\% NA): inspired concentration of oxygen on ventilatory support (the higher the more critical; Ventilation = 0: no ventilatory support) (h).
    \item \textit{Neurosurgery.day0} (categorical, binary variable, no missing values): neurosurgical intervention performed on day of admission (h).
    \item \textit{IGS.II} (continuous, positive values, 2\% NA): Simplified Acute Physiology Score (h).
    \item \textit{Tranexomic.acid} (categorical, binary variable, no missing values): administration of the tranexomic acid (h). 
    \item \textit{TBI} (categorical, binary variable, no missing values): indicates if the patient suffers from a traumatic brain injury (h).
    \item \textit{IICP} (categorical, binary variable, 70.9\% NA): at least one episode of increased intracranial pressure; mainly in traumatic brain injury; usually associated with worse prognosis (h).
    \item \textit{EVD} (categorical, binary variable, no missing values): external ventricular drainage (EVD); mean to drain cerebrospinal fluid to reduce intracranial pressure (h). 
    \item \textit{Decompressive.craniectomie} (categorical, binary variable, no missing values): surgical intervention to reduce intracranial hypertension (h).  
    \item \textit{Death} (categorical, binary variable, no missing values): death of the patient  (h).
    \item \textit{AIS.head, AIS.face} (ordinal, discrete, integers between 0 and 6 and 4 1.7\% NA \& 1.7\% NA): Abbreviated Injury Score, describing and quantifying facial and head injuries (AIS = 0: no injury; the higher the more critical) (h).
    \item \textit{AIS.external} (continuous, discrete, integers between 0 and 5, 1.7\% NA): Abbreviated Injury Score for ex- ternal injuries, here it is assumed to be a proxy of information avail- able/visible during pre-hospital phase (ph/h).
    \item \textit{ISS} (continuous, discrete, integers between 0 and 75, 1.6\% NA): Injury Severity Score, sum of squares of top three AIS scores (h).
    \item \textit{Activation.HS.procedure} (categorical, binary variable, 3.7\% NA): activation of hemorragic shock procedure in case of HS suspicio (h).
    \item \textit{TBI\_Death} (categorical, binary variable, no missing values): death of the patients suffering from a traumatic brain injury (h).
    \item \textit{TBI\_Death\_30d} (categorical, binary variable, no missing values): death of the patients suffering from a traumatic brain injury in the 30 days (h).
    \item \textit{TBI\_30d} (categorical, binary variable, no missing values): traumatic brain injury in the 30 days (h).
    \item \textit{Death\_30d} (categorical, binary variable, no missing values): death in the 30 days (h).
    \item \textit{Shock.index.ph} (continuous, positive values, 30.5\% NA): ratio of heart rate and systolic arterial pressure during pre-hospital phase (ph).
    \item \textit{majorExtracranial} (categorical, binary variable, no missing values): major extracranial lesion (h). 
    \item \textit{lesion.class} (no missing values): partition given by the doctors with $K=4$ classes: axonal, extra, other, intra.
    \item \textit{lesion.grade} (no missing values): partition given by the doctors with $K=3$ classes: high, low, other.
\end{itemize}

\begin{figure}
\hspace{-3cm}
\includegraphics[scale=0.5]{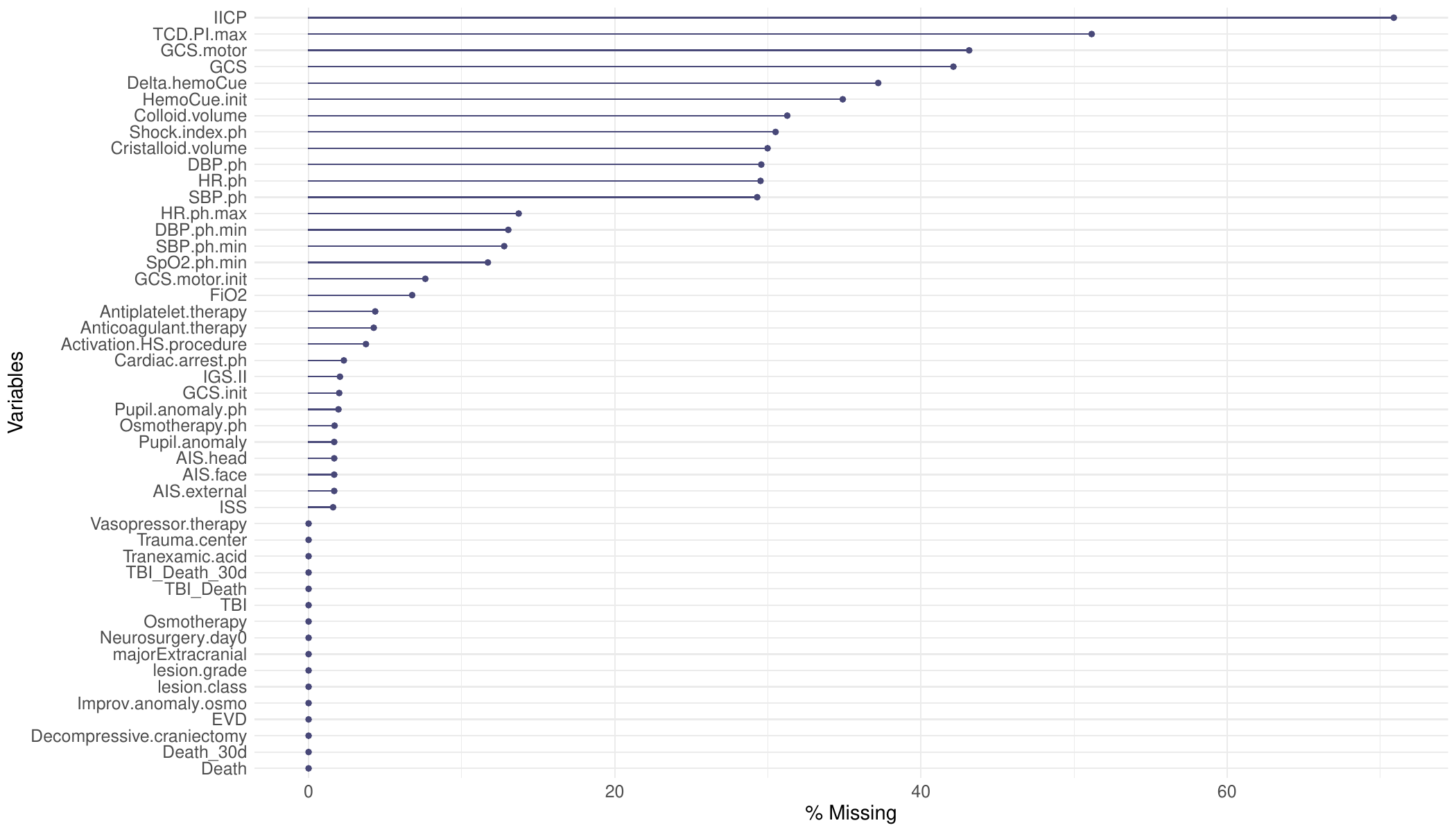}
\caption{Percentage of missing values per variable for the Traumabase dataset.}
\label{fig:percNA}
\end{figure}

\end{appendices}


\bibliography{sn-bibliography}% common bib file

\end{document}